\documentclass{article}
\usepackage{microtype}
\usepackage{graphicx}
\usepackage{booktabs} 

\usepackage{times}
\usepackage{float}
\usepackage{amsmath, amssymb}
\usepackage[utf8x]{inputenc} 
\usepackage{hyperref}
\usepackage{amsthm}
\usepackage{subcaption}

\newtheorem{lemma*}{Lemma}
\newtheorem{theorem*}{Theorem}
\newtheorem{theorem}{Theorem}
\newtheorem{lemma}{Lemma}
\newtheorem{corollary}{Corollary}

\newif\ifboldnumber

\usepackage{amsmath}
\usepackage{xparse}

\ExplSyntaxOn
\NewDocumentCommand{\RN}{m}
 {
  \textup{ \int_to_Roman:n { #1 } }
 }
\ExplSyntaxOff

\usepackage[accepted]{icml2020}

\icmltitlerunning{Adaptive Distributed Stochastic Variance Reduced Gradient for Heterogeneous Distributed Datasets}
\begin{document}

\twocolumn[
\icmltitle{Adaptive Distributed Stochastic Variance Reduced Gradient for Heterogeneous Distributed Datasets}

\icmlsetsymbol{equal}{*}
\begin{icmlauthorlist}
\icmlauthor{ilqar Ramazanli}{equal,to}
\icmlauthor{Han Nguyen}{equal,to}
\icmlauthor{Hai Pham}{equal,ed} 
\icmlauthor{Sashank J. Reddi}{goo}
\icmlauthor{Barnab\'{a}s P\'{o}czos}{ed} \\
\end{icmlauthorlist}

\icmlaffiliation{to}{Department of Mathematical Sciences, Carnegie Mellon University, Pittsburgh, PA}
\icmlaffiliation{goo}{Google Research, New York, NY}
\icmlaffiliation{ed}{School of Computer Science, Carnegie Mellon University, Pittsburgh, PA}


\icmlcorrespondingauthor{ilqar Ramazanli}{iramazan@andrew.cmu.edu}
\icmlcorrespondingauthor{Han Nguyen}{hann1@andrew.cmu.edu}
\icmlcorrespondingauthor{Hai Pham}{htpham@cs.cmu.edu}

\icmlkeywords{Machine Learning, ICML}

\vskip 0.3in
]

\printAffiliationsAndNotice{\icmlEqualContribution} 

\begin{abstract}
We study distributed optimization algorithms for minimizing the average of \emph{heterogeneous} functions distributed across several machines with a focus on communication efficiency. In such settings, naively using the classical stochastic gradient descent (SGD) or its variants (e.g., SVRG) with a uniform sampling of machines typically yields poor performance. It often leads to the dependence of convergence rate on maximum Lipschitz constant of gradients across the devices. In this paper, we propose a novel \emph{adaptive} sampling of machines specially catered to these settings. Our method relies on an adaptive estimate of local Lipschitz constants base on the information of past gradients. We show that the new way improves the dependence of convergence rate from maximum Lipschitz constant to \emph{average} Lipschitz constant across machines, thereby, significantly accelerating the convergence. Our experiments demonstrate that our method indeed speeds up the convergence of the standard SVRG algorithm in heterogeneous environments.
\end{abstract}

\section{Introduction}

In this paper, we study distributed optimization algorithms to solve finite-sum problems of the form: 
\begin{equation}\label{eq1}
\min_{x\in\mathbb{R}^d}F(x):=\frac{1}{M}\sum_{m=1}^{M}F_m(x),
\end{equation}
where $F_m(x)=\frac{1}{n}\sum_{j\in S_m}f_j(x)$, $\cup_{m=1}^MS_m=\{1,\ldots,N\}$, $|S_m|=n$ and all sets are disjoint. Here $\{S_m\}_{m=1}^M$ represents partitions  of a large dataset with $N$ data points such that each dataset $S_m$ only contains $n \ll N$ data points (in fact, it is admissible to assume that $|S_m|$ varies across the workers and the analysis of our method would stay the same). 
This problem arises naturally in machine learning in the form of empirical risk minimization. We are particularly interested in the decentralized distributed learning setting where each $S_m$ is stored locally in a worker. In this setting, each function $F_m$ is a local average of the total average loss function $F$. We aim to minimize the total average loss function with minimal communication amongst the workers.\medskip\\
Traditional distributed machine learning settings assume that each worker $S_m$ has independent and identical distributed (\emph{i.i.d.}) samples from an underlying distribution. 
This implicitly implies that each of the local average loss function $F_m$ is statistically similar to the total average loss function $F$ due to the law of large numbers. 
In contrast, we assume that the data on each worker may be generated from different distributions. 
Consequently, the local average loss functions can be potentially very different from each other and from the total average loss function. 
A typical example of this setting is one where a large dataset is gathered to a server and then distributed unevenly to all workers in the sense that each worker only contains some main features of the whole data. Another canonical example of our setting is that of learning a machine learning model using data from mobile phone users. Here each mobile phone user is a worker and contains data such as photos, texts based on their interest. As a result, the characteristics of data on each mobile phone vary by user. Our setting is a particular case of a more general framework, Federated Learning \cite{konevcny2016federated}, which is a challenging and exciting setting for distributed optimization. \medskip\\
In the settings above, the change of the gradients from some worker's local functions could dominate the change of the gradient of the global function $\ F$. We refer to these workers as informative workers. In particular, the gradients of some workers might change very slowly so that their contribution to the change of the gradient of $F$ is almost negligible. Hereafter, we refer to such workers as non-informative workers. Naively using SGD or its variance reduced variants (eg., SVRG)  with uniform sampling often yields poor performance in such settings because the majority of the computation is spent on non-informative workers. This insight was exploited in the work of \cite{chen2018lag} to prevent computing new gradients of non-informative workers frequently in the deterministic gradient descent (GD). We can think of our work as their stochastic counterpart..\medskip\\ 
Our primary goal in this paper is to design an adaptive sampling strategy for SVRG.
It's a reduced variance variant of SGD that works efficiently in the heterogeneous setting of our interest by paying more attention to informative workers. 
We want to emphasize in an environment that the information held at each worker may be very different. 
Treating them, in the same way, may results in inefficiency due to loss of information.
For instance, using uniform distribution to select workers as in the standard SGD and SVRG slows down since it keeps revisiting non- informative workers.
Formally, since the gradients of non-informative workers are very small comparing to the gradients of informative workers, the optimization will have very small (or almost zero) improvement by following these directions. Thus, it is desirable to design an adaptive optimization method that is able to select useful workers during the training process. By selecting workers actively, we are able to save a number of iterations from reaching a predetermine precision comparing to the uniform based sampling method.\medskip\\
\textbf{Contributions.} In light of the above discussion, we state the  main contributions of this work. 
\begin{itemize}

\item First, we develop an \emph{adaptive} sampling strategy for the  SVRG algorithm and show that it improves the convergence of the SVRG algorithm in the heterogeneous setting. 
Our method is also robust to the homogenous data across machines; meanwhile, few machines have outlier data with much larger Lipschitz constant.
In detail, our adaptive sampling technique pays more attention to informative workers. Consequently, we can reduce the dependency on the maximum of the Lipschitz constants to the average of them in the convergence rate of the SVRG algorithm.
Besides, our experiments show that our adaptive algorithm is more stable with large step sizes than the standard SVRG algorithm. 
\item Second, we design an efficient adaptive local Lipschitz estimation method that is another version of the importance sampling algorithm due to \cite{xiao2014proximal}.
Our method outperforms the result above in the sense that we don't need any pre-information regarding the exact or estimated values of Lipschitz constants.
We provide a robust theoretical analysis of the estimation method and show that the convergence rate of this method is almost the same as the importance sampling strategy. 
\item  Third,  we propose a new parallel communication method with optimal cost.
This method enables sampling with respect to weights in a condition that initially, machines know just their weights.
In detail, we show that our parallel sampling technique can choose $R$ workers by just using $O(M)$ many worker-worker communications for any $R$. 
\end{itemize}

\section{Related Work} 

\paragraph{Single-machine Setting:} Although there were some efficient SGD-based approaches for the single-machine setting \cite{bottou2010large,robbins1951stochastic}, none of them did better than sub-linear convergence rate, leading to SVRG \cite{johnson2018training} and others \cite{le2013stochastic, defazio2014saga,bouchard2015online,zhao2015stochastic} that addressed variance reduction and hence improving the convergence rate. Serving the same purpose, gradient-based approximate sampling methods \cite{alain2015variance,katharopoulos2017biased,katharopoulos2018not} were proposed, but they suffered from high computation cost. To solve this problem, more robust and less computation-consuming methods based on gradient norms \cite{johnson2018training, stich2017safe} were used to reduce the sampling cost while still maintaining variance reduction goal. 

\vspace{-3mm}

\paragraph{Distributed Learning:} Distributing large-scale datasets across multiple servers is an effective solution to reduce per-server storage and memory utilization \cite{dean2008mapreduce,zaharia2010spark,dean2012large}. The first and traditional approach is synchronous parallel minibatch SGD \cite{dekel2012optimal,li2014communication}. 
Although being able to split the workloads to many nodes to speed up jobs, this method suffers from the high latency problem which might happen due to one or some slow nodes, which can be solved by the second group of asynchronous methods \cite{recht2011hogwild, reddi2015variance,duchi2013estimation}. 

\vspace{-3mm}

\paragraph{Communication Efficiency:}
In order to overcome the communication burden in distributed optimization, communication-efficient methods have been proposed \cite{zinkevich2010parallelized, zhang2013information, zhang2012communication, shamir2014communication,reddi2016aide,chen2018lag}. 
The methods by \cite{zinkevich2010parallelized,shamir2014communication,reddi2016aide} reduce the communication rounds by increase the computation on local workers. However, those approaches also assumed \textit{i.i.d} setting, unlike ours. 
On the contrary, the work by \cite{chen2018lag} tackles with the \textit{non-i.i.d} setting. Specifically, they propose an algorithm that can detect slow-varying gradients and skip their calculations when computing the full gradients to reduce the communication cost. 

\section{Preliminaries}

\paragraph{Notations:}Standard inner product and $\ell_2$ norm induced from that are denoted by $\left< ., . \right>$ and $\|.\|$ correspondingly. $\mathbb{E}[.|X]$ and $\mathbb{E}[.]$ stands for conditional and full expectations. Sets \{1,2,\ldots, M\},\{N,N+1,\ldots,M\} and \{N,2N,3N,\ldots, MN\} will be represented by [M], [N,M] and [M]N respectively.

\paragraph{Problem Setup: }
We consider the finite sum optimization problem (\ref{eq1}) in the distributed learning setting where each function $F_m$ is stored on a local worker. 
We assume workers can communicate with each other and also with the server. 
However, each type of communication has its own cost.  
In practice, servers can perform mass broadcasting to multiple workers, but not vice-versa~\cite{chen2018lag}. 
Therefore, we assume that server to worker communication is cheaper than worker to worker. However,  worker to server communication is more expensive.
Therefore, the cost of information flow is dominated by worker to server and worker to worker communications.
For convergence analysis, we assume that each function $F_j$ is convex with $L_j$-Lipschitz gradients. 
In other words, for any $x,y\in\mathbb{R}^d$ and $j\in[M]$, we have the following:
\vspace{-1mm}
\begin{equation*}
F_j(y)- F_j(x)-\left<\nabla F_j(x),y-x\right>\le \frac{L_j}{2}\|y - x\|^2.
\end{equation*}
Moreover, we use $\bar{L}$ to denote the average of Lipschitz constants and $\widetilde{L}$ for a maximum of them. Due to the \emph{non-i.i.d} data distributed setting, we assume that $L_j$'s vary highly across the workers.  We assume each $F_j$ is $\lambda_j$-strongly (and $F$ is $\lambda$-strongly) convex, i.e. for any $ x,y\in\mathbb{R}^d$  we have :
\vspace{-1mm}
\begin{equation*}
    F_j(y)-F_j(x)-\left<\nabla F_j(x),y-x\right>\ge\frac{\lambda_j}{2}\|y-x\|^2.
\vspace{-2mm}
\end{equation*}
Finally, we denote  $K_j = \sup_{i \in S_j} \frac{l_i}{\lambda_j}$ and assuming it's not tremendously big, where $l_i$'s are Lipschitz constants for gradient of atomic functions $f_i$.

\vspace{-3mm}

\paragraph{Motivation:}

The main mechanism in large scale optimization for machine learning is the SGD algorithm. At each iteration, this method picks a function $F_m$ uniformly random then uses the gradient of this chosen function in the gradient descent update instead of the full gradient. Although the computation is saved, the convergence rate of the SGD algorithm  depends strongly on the variance of the stochastic gradients.  \cite{bottou2018optimization} shows that if $\mathbb{E}_m\|\nabla F_m(x)\|^2\le M+M_G\|\nabla F(x)\|^2$ holds for some positive constants $M, M_G$, then following statement is satisfied. \medskip\\
\textit{\textbf{Theorem}:
If we choose the step size $\alpha_t=\frac{\beta}{\sigma +t}$ for some $\beta>\frac{1}{c\lambda}$ and $\sigma>0$ such that $\alpha_1<\frac{\lambda}{LM_G}$, then for all $t\in\mathbb{N}$, the expected optimality gap satisfies}
$$\mathbb{E}[F(x_t)-F(x^*)]\le\frac{\nu}{\sigma+t}$$
\textit{where}
$\nu =\max\left\{\frac{\beta^2L M}{2(\beta c\lambda-1)},(\sigma+1)[F(x_0)-F(x^*)]\right\}$
\textit{and $L$ is the Lipschitz constant of $\nabla F$}. 

The convergence rate above is sublinear, and it also depends on the Lipschitz constant of the gradient, which may be very big in practice. 
Therefore the SVRG algorithm was proposed in \cite{johnson2013accelerating} to overcome these issues.
\begin{algorithm}[ht]
\caption{\textbf{SVRG} with fixed sampling \cite{xiao2014proximal}.}
{\bfseries Input:}  Initial solution $\bar{x}_0$, step size $\eta$,  number of iterations on each epoch $T$ and the number of epochs $K$. Provided fixed distribution $p=[p_1,\ldots,p_M]$ to sample the indices.

\begin{algorithmic}[1]
\FOR{$k=1$ {\bfseries to} $K$} 
\STATE Compute $\nabla F(\bar{x}_{k-1})$
\STATE $x_0=\bar{x}_{k-1}$
\FOR {$t=1,\ldots, T$}
\STATE Randomly pick $m_t\in\ [M]$ w.r.t. distribution $p$ 
\STATE  
 $v_t=\frac{\nabla
F_{m_t}(x_{t-1})}{M p_{m_t}}-\frac{\nabla F_{m_t}(\bar{x}_{k-1})}{M p_{m_t}}+\nabla F(\bar{x}_{k-1})$
\STATE Update $x_{t}=x_{t-1}-\eta v_t$\\
\ENDFOR
\STATE Update $\bar{x}_{k}$ by choosing uniformly random $\{x_{t}\}_{t=0}^{T-1}$
\ENDFOR
\end{algorithmic}
{\bfseries Output: } $\bar{x}_K$
\end{algorithm}

At each inner iteration$-t$ the algorithm chooses a function $F_{m_t}$ according to the distribution $p$ and constructs an unbiased estimation $v_t$ of the gradient $\nabla F(x_{t-1})$. 
Similarly to the SGD algorithm, the convergence rate of this method is then affected by the term $\mathbb{E}\|v_t\|^2$. 
Intuitively, the smaller values of this quantity give a better rate. 
After algebraic manipulations, we conclude the following equation:
\begin{align}
     \label{eq2}
    \mathbb{E}  \|v_t\|^2 = &\sum_{m=1}^M\frac{\|\nabla F_m(x_{t-1})-\nabla F_m(\bar{x}_{k-1})\|^2}{M^2 p^2_m}    \\
     &+\|\nabla F(x_{t-1})\|^2-\|\nabla F(x_{t-1})-\nabla F(\bar{x}_{k-1})\|^2 \hspace{1mm} \nonumber
\end{align}
Notice that the first term above depends on the distribution $p$, which means that the choice of $p$ has some effect on the convergence rate. One standard option in practice of $p$ is the uniform distribution.
In this case, the standard analysis of the SVRG algorithm  \cite{johnson2013accelerating} shows that:
 \begin{align*}
      \mathbb{E}[F(\bar{x}_{k})&-F(x^*)] 
     \le \\
     &\left(\frac{1}{\lambda\eta T(1-2\eta\widetilde{L})}+\frac{2\eta \widetilde{L}}{1-2\eta\widetilde{L}}\right)[F(\bar{x}_0)-F(x^*)].
 \end{align*}
 We notice that his rate depends on the maximum Lipschitz constant $\widetilde{L}$.  
 Although this rate is better than previous methods, being dependent on $\widetilde{L}$ can be inefficient when data being non-iid  distributed.
Especially because in the distributed machine learning setting since it may cause many communications. Given that our goal is to design a communication-efficient algorithm, it is crucial to reduce this dependency. 
 \cite{xiao2014proximal} proposed a solution to this problem if Lipschitz constants of the gradients are previously known by setting fixed distribution above to $p=\left[\frac{L_1}{\sum L_m},\ldots,\frac{L_M}{\sum L_m}\right]$. 
They showed that the convergence rate is as following:
  \begin{align*}
     \mathbb{E}[F(\bar{x}_{k}) &-F(x^*)] \le \\   
       &\left(\frac{1}{\lambda\eta(1-2\eta\bar{L})}+\frac{2\eta \bar{L}}{1-2\eta\bar{L}}\right)[F(\bar{x}_0)-F(x^*)]
 \end{align*}
 which depends on a smaller constant$-$average of Lipschitz constants  $\bar{L}$.

\section{Theoretical Results}

In this section, we discuss the details of the theoretical contributions of this work.
First, we provide intuition behind the estimation of local Lipschitz constants of each machine.
Second, we provide the main algorithm that uses the idea of local Lipschitz values to extend it to the Adaptive SVRG algorithm.
Third, we provide our Novel Sampling Strategy, and at last, we present tools that we used for the proof of the main algorithm.

One crucial question arises on the method due to \cite{xiao2014proximal} is what if we don't have access to any information about the exact or estimated value of Lipschitz constants.
Should we return to uniform sampling, or are there alternative methods to solve this issue.
Given that the maximum of Lipschitz constants can be drastically different from their average return to uniform sampling will give us a prolonged convergence rate.
Estimating Lipschitz constants by querying many points before executing the algorithm can be very slow as the estimation process can take exponential runtime with respect to dimension $d$.\medskip\\
Therefore, we suggest a solution which estimates local Lipschitz constants efficiently and prove that the algorithm is still converging as fast when we sampling happens with these weights.
Going back to equation (\ref{eq2}), we notice that choosing a distribution that minimizes the first summand adaptively at each iteration will improve the performance of the SVRG algorithm. 
To clarify this issue, we analyze the following optimization problem.
 \begin{align*}
     \min_{p\in\Delta_M}\sum_{m=1}^M\frac{\|\nabla F_m(x_{t-1})-\nabla F_m(\bar{x}_{k-1})\|^2}{p^2_m}.
 \end{align*}
 By applying the KKT conditions, the solution of the above 
 problem is 
 \begin{align*} 
     p^{k,t}_m&=\frac{\|\nabla F_m(x_{t-1})-\nabla F_m(\bar{x}_{k-1})\|}{\sum_{m=1}^M\|\nabla F_m(x_{t-1})-\nabla F_m(\bar{x}_{k-1})\|}
 \end{align*}
This probability distribution does not depend on any information but local values of the function, which is easily accessible.
The only requirement here is computing the following rephrase of the  $\nabla F_m(x_{t-1})- \nabla F_m(\bar{x}_{k-1})$ which can be rewritten as the following:
\begin{align} \label{aveeachmachine}
 \frac{1}{|S_m|}\sum_{j\in S_m} \Big( \nabla f_j(x_{t-1}) - \nabla f_{j}(\bar{x}_{k-1})  \Big)
\end{align}
However, naively computing each of these values is an expensive task as it requires to go through each datapoint once.
To overcome this issue, we first propose an efficient estimation method to this expression, then we prove that weights due to estimations also successfully give a fast convergence rate.

In the following lemma (extension of \citet{concentrationwithoutreplacement}), we show that taking a small subsample  $\widetilde{S}_m \subset S_m$ at each machine and computing the average :
\begin{align} \label{esti}
 \frac{1}{|\widetilde{S}_m|}\sum_{j\in \widetilde{S}_m} \Big( \nabla f_j(x_{t-1}) - \nabla f_{j}(\bar{x}_{k-1})  \Big)
\end{align}
of this sample in this machine will successfully estimate the expression in (\ref{aveeachmachine}).
Setting $a_i = \nabla f_i(x_{t-1}) - \nabla f_{i}(\bar{x}_{k-1})$ below lets us to bound $\frac{\|b-a\|}{\|\mu\|}$ with respect to $K_j$ and this helps us to use the lemma \ref{lmm1}.

\begin{lemma} \label{lmm1}
Let $S = \{ a_1, a_2, \ldots , a_N \}$ be a set of vectors that $a_i \in \mathbb{R}^d$,  $a \leq a_i \leq b$ for any $i \in [N]$ and fixed vectors $a,b \in \mathbb{R}^d$.
$\mu$ denotes the average of vectors in S: $\mu = \frac{1}{N} \sum_{i=1}^N a_i$ and
$X_1,X_2,...,X_n$ is the set of size $n$ that uniformly sampled without replacement from $S$. 
Given that $n = \frac{1}{\tau^2} \frac{\|b-a\|^2}{2 \mu^2} \log{\frac{2d}{\delta}}$ then the 
following inequality is satisfied with probability at least $1- \delta$ :
\begin{align*} 
     \left\|\frac{1}{n} \sum_{i=1}^n X_i - \mu \right\|_2 \leq \tau \| \mu \|_2
\end{align*}
\end{lemma}

As an illustration  let's look at the example that $\tau$ selected as $\tau = 0.05$.
This value of $\tau$  corresponds that each of the weights has an estimation of $\pm 5\%$ error rate.
For instance, estimation changes initial weights as:
\begin{align*}
w=(40, 40, 60, 60)   \implies \widetilde{w}=(39, 41, 58, 61)    
\end{align*}
then categorical probabilities change as:
\begin{align*}
p=(0.2, 0.2,0.3,0.3)   \implies  \widetilde{p}=(0.195, 0.206,0.291,0.306).
\end{align*}
We can easily show the probability of each category roughly cannot change more than $2 \tau$ times.
To show this phenomenon in the example above we rewrite $\widetilde{p}$ as
\begin{align*}
 (0.195, 0.206,0.291,0.306) &= 0.97(0.2,0.2,0.3,0.3)\\
 &+0.03(003,0.4,0,0.57)
\end{align*}
More generally, $ \widetilde{p} = 0.97p +0.03q $ where $q$ is another categorical distribution.
To generalize this decomposition, we prove the following lemma:

\begin{lemma} \label{noisydist}
Let $\mathcal{P}$ be a categorical distribution with weights $(w_1,w_2,...,w_m)$ and
$\widetilde{\mathcal{P}}$ be a perturbation of $\mathcal{P}$ with new weights  $(w_1+\delta_1,w_2+\delta_2,...,w_m+\delta_m)$. 
Lets $\mathcal{Q}$ be a categorical distribution with  weights 
$(\delta_1 - w_1 \mathrm{min}(\frac{\delta_i}{w_i}), 
\delta_2 - w_2 \mathrm{min}(\frac{\delta_i}{w_i}),  \ldots ,
\delta_M - w_m \mathrm{min}(\frac{\delta_i}{w_i}))$ and $\lambda$ defined as
\begin{equation*}
\gamma = 1- \min_{1 \leq i \leq m}  \frac{ \frac{w_i+\delta_i}{ w_1+\delta_1 + w_2+\delta_2 + \ldots + w_m+\delta_m
} }{ \frac{w_i}{ w_1+w_2+ \ldots + w_m }  }
\end{equation*}
Then, we can decompose $\widetilde{\mathcal{P}}$ to the combination of $\mathcal{P}$ and $\mathcal{Q}$ as following:
\begin{align*}
    \Psi = 
    \begin{cases}
        \text{sample with respect to } \mathcal{P} & \text{with probability } 1- \gamma \\
        \text{sample with respect to } \mathcal{Q} & \text{with probability } \gamma
    \end{cases}
\end{align*}
\end{lemma}

\begin{figure*}
\centering
\includegraphics[width=1.0\linewidth]{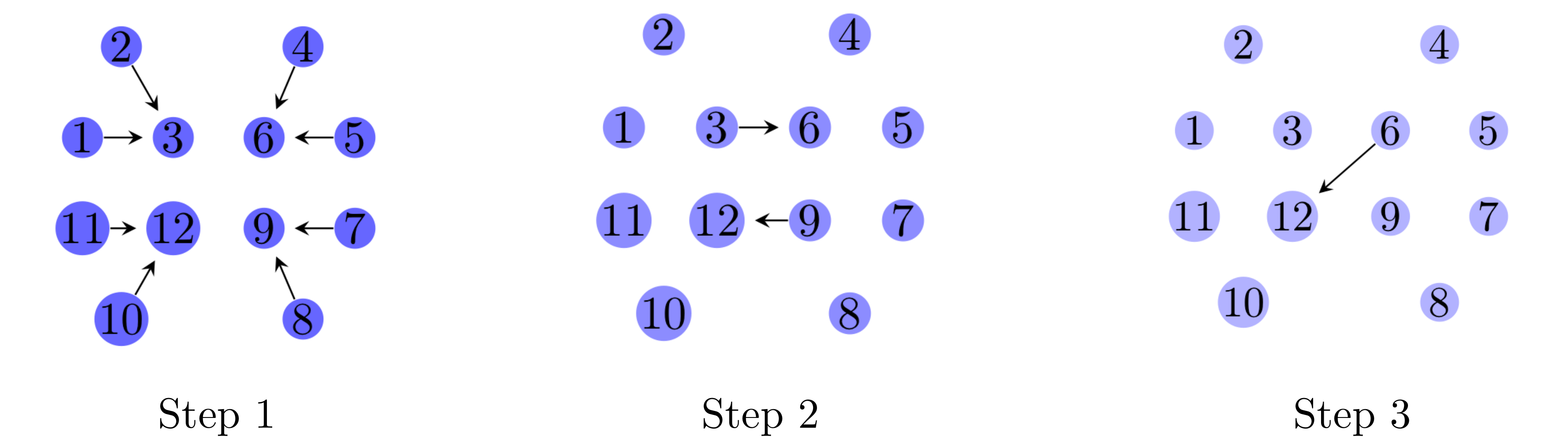}
\caption{Illustration of the Parallel Communication algorithm with the number 12 nodes (M=12) and the parameter $R$ to be 3.
Communication due to line 1 is presented in Step 1 :
Machines $\{1,2\}$ send estimated weight to the machine 3, ( 3 sends to itself, that's why no need to mention),  $\{4,5 \}$ send to 6, $\{7,8 \}$ send to 9, and $\{10,11 \}$ send to 12.
Communication due to the loop of line 3 is given in Step 2 and Step 3. Note that $\log_2{\frac{12}{3}}=\log_2{4} = 2$ and that's why there is two iterations--communication steps here.
First of these communications happen between machine $\{3,6\}$ and $\{9,12\}$ in parallel. 
For the set \{3,6\}, machine 6 is the receiver, and the machine $6-2^0 3=3$ is the sender, and for the set \{6,12\}, the machine 12 is the receiver, and the machine $12-2^0 3=9$ is the sender. 
For the second iteration, we have just machines $\{6,12\}$ participating in communication where $12$ is the receiver; meanwhile, $12-2^1 3 =6$ is the sender.
}
\label{fig:communication}
\end{figure*}

\subsection{Communication Algorithm}

In the previous section, we already discussed how to efficiently estimate weight--the importance of each machine at a given time.
The next important question is how to deliver this information among machines, so we can successfully sample essential machines.
Given that worker to server communication is more expensive than other types of communications, sending all weights to the server directly should be avoided.

Therefore, it's intuitive that we need to use the communication among workers to support the sampling process.
Hence, we target to design a communication method, which optimizes the number of bytes transferred meanwhile having an efficient runtime. 
To provide an intuition for the problem, we present a method of how to sample from the set \{2,3,5,7\} with weights  \{1,1,3,2\} efficiently.

We have four machines, and each of them carries one of the prime numbers above with corresponding weight.
The idea is as simple as the following.
The machine one sends its information (number 2 and weight 1) to machine two, and latter samples among prime numbers {2,3} with respect to weights $1$ and $1$.
 Meanwhile, machine three sends information (number 5 and weight 2) to machine four, and the same process happens there as well.
In the second phase, machine two sends its sampled number to machine four together with cumulative weight from the first phase $(1+1)$. 
Then machine four makes final sampling with weights $(1+1)$ and $(3+2)$ and announces the final result.
The probabiliy of selection of number 7 at the end is equal to $\frac{2}{3+2} \times \frac{3+2}{3+2+1+1} = \frac{2}{7}$ which is the desired probability. 
A simple analysis of this idea, tells us this method runs in $\mathcal{O}(\log{M})$ time using $\mathcal{O}(M)$ bytes transferred.

Note that, after each iteration/update, weights of each machine will change very incrementally.
Therefore recomputing weights and resampling every time is not efficient.
That's why we extend the idea above to enable the sampling of many machines.
A natural extension is sending $R$ many information machines at each step instead of one, which enables sampling of $R$ machines at the end. 
The transfer complexity of this method would  $\mathcal{O}(RM)$ bytes.
However, in the following algorithm, we show how to perform this task using still transfer - $\mathcal{O}(M)$ bytes, no matter how many machines we sample.
We give an illustration for $R=3$ in the figure above and provide analysis in the lemma below.

\begin{algorithm}[H] \label{algo2}
\caption{  \textbf{PC:} \hypertarget{pc}{Parallel} Communication }

{\textbf{Input:}} weights $\{{w}_1, {w}_2,\ldots {w}_M \}$ and group size $R$

\begin{algorithmic}[1]
\STATE  Machine $m \hspace{-.5mm}\in\hspace{-.5mm}[M]$ sends $(m,w_m)$ to machine $\lceil \frac{m}{R} \rceil R$
\STATE Worker $m \in [\frac{M}{R}]R $  samples $R$ indices with replacement from the interval $[m-R+1,m]$ with respect to weights $\{w_{m-R+1},\ldots,w_{m}\}$ and assigns them to $i^m = (i_1^m,i_2^m,...,i_R^m)$ and set
$w_m = \sum_{j=m-R+1}^{m}w_{j}$.

\FOR{$h \in [\log{\frac{M}{R}}] $ }
\STATE  For $m \in [\frac{M}{2^{h}R }]2^{h}R$  denote $s_m = m - 2^{h-1}R$.
 \STATE  Worker $s_m$ sends $(i^{s_m},w_{s_m})$ to  worker $m$.
 \STATE  For any  $j\in[ R]$ worker $m$ samples from $\{ i_j^{s_m}, i_j^{m}\}$ with weights $\{w^{s_m}, w^{m} \} $ and assigns result to $i_j^{m}$ 
\STATE  $w^{m} \leftarrow w^{m}+w^{s_m}$
\ENDFOR
\end{algorithmic}
{\textbf{ Output:}} histogram of $i^M$
\end{algorithm}

\vspace{-2mm}

\normalsize
\begin{lemma}\label{communication}
The Parallel Communication sampling technique above samples $R$ many workers with replacement using just $O(M)$  worker to worker communication for any $R$. 
Furthermore, sampling process ends in total time of $\mathcal{O}(R\log{M})$.
\end{lemma}


\subsection{Main Algorithm}

In this section, we merge all the ideas discussed in previous sections to build an adaptive Distributed SVRG algorithm.
Our algorithm outperforms previous algorithms under the condition that: 
i) there is no pre-information regarding Lipschitz constants and 
ii) the maximum of the Lipschitz constants is much higher than the average of them.

We use the estimation discussed in expression \ref{esti} and lemma \ref{lmm1} to efficiently approximate the local Lipschitz constant (weight) of each worker in line 8.
Then, using algorithm \hyperlink{pc}{$\mathbf{PC}$} we transfer this information among workers and perform sampling in the next line. 
Finally, using the lemma \ref{noisydist} and lemma \ref{mainthm} we complete the analysis of the proposed algorithm.

\begin{algorithm}[H]
\caption{\textbf{ASD-SVRG:} \hypertarget{asdsvrg}{Adaptive}  Sampling Distributed SVRG }
 {\textbf{Input:}}  Initial solution: $\bar{x}_0$ , step length: $\eta$, outer loop size: $K$ and inner loop size: $T$, sampling size-$R$ 
 
\begin{algorithmic}[1]
\FOR{$k=1$ {\textbf {to}} $K$}
\STATE The server distribute $\bar{x}_{k-1}$ to all workers
\STATE \textbf{In parallel}: Worker $m$ computes $\nabla F_m(\bar{x}_{k-1})$ and 
sends it  to the server (for any $m\in[M]$) \\
\STATE Server computes $\nabla 
F(\bar{x}_{k-1})$ and sends it to  workers\\
\STATE $x_0=\bar{x}_{k-1}$
\FOR{$t=1$ {\textbf {to}} $T$}
\STATE Server sends $x_{t-1}$ to all workers
\STATE \textbf{In parallel}: Uniformly sample $\widetilde{S}_m \subset S_m$ of size $n_m$.
Compute estimation $\widetilde{w}_m$ of $w_m$ using $\widetilde{S}_m$  $\widetilde{w}_m = \| \frac{1}{n_m}\sum_{j\in \widetilde{S}_m} \nabla f_j(x_{t-1}) - \nabla f_{m}(\bar{x}_{k-1})\|$  

\vspace{1mm}

\STATE $\mathrm{H} =\hyperlink{pc}{\mathbf{PC}}((\widetilde{w}_1,\widetilde{w}_2,\ldots,\widetilde{w}_m) , R)$
\STATE \textbf{In parallel}: For any  worker $m \in \mathrm{H}$ compute :\\ 

\vspace{1mm}
  
$v^m_t=\frac{\nabla  F_{m}(x_{t-1})}{M \widetilde{p}^{k,t}_{m}}-\frac{\nabla F_{m}(\bar{x}_{k-1})}{M \widetilde{p}^{k,t}_{m}}+\nabla F(\bar{x}_{k-1})$ \textit{where} $\widetilde{p}^{k,t}_{m}=\frac{\widetilde{w}_m}{\sum_{m=1}^M \widetilde{w}_m}$

\vspace{0.5mm}

\STATE Update $x_t = x_{t-1} - \eta \sum_{m=1}^{M} \frac{\mathrm{H}[m] \times v^m_t}{R}$
\STATE Send $x_t$ back to the server
\ENDFOR
\STATE Update $\bar{x}_{k}$ by choosing uniformly random $\{x_{t}\}_{t=0}^{T-1}$
\ENDFOR 
\end{algorithmic}
{\textbf{ Output:}} $\bar{x}_K$ 
\end{algorithm}

In the following theorem, we are characterizing the convergence rate of the algorithm described above.
Notice that the convergence rate of our method also depends on the average of Lipschitz constants.
Thus, our method is better than uniform sampling SVRG.
Moreover, in comparison with the importance sampling method \cite{xiao2014proximal}, we can see that our algorithm is still at least as good as that method.

The sub-sampling method in line 8 is crucial when each of the machines has tremendous data.
Even though it gives some small error to sampling weights, we show that it does not affect the convergence rate importantly.

\begin{theorem} \label{thm2}
Given $K,T,R>0$, and $\eta$ small. The iteration $\bar{x}_k$ in \hyperlink{asdsvrg}{$\mathbf{ASD-SVRG}$} converges to the optimal solution $x^*$ linearly in expectation. Moreover, under the condition each of the $n_m$ satisfies the condition in lemma \ref{lmm1}, then the following inequality get satisfied:
\begin{align*}
  & \mathbb{E}[F(\bar{x}_{k})-F(x^*)]\le \rho\mathbb{E}[F(\bar{x}_{k-1})-F(x^*)],
\end{align*}
with probability of $1-\delta$ where $\rho$ defined as
\begin{align*}
  \rho=\frac{1}{\lambda T \eta \Big(1-\eta ( 1 +\frac{2+5\tau}{R} )\bar{L}\Big)} 
  +  \frac{\eta\frac{2+5\tau}{R} \bar{L}}{1-\eta (1 + \frac{2+5\tau}{R})\bar{L}}.
\end{align*}
\end{theorem}

\begin{proof}[\textit{Proof Sketch:}] 
To prove the convergence rate above, we first find a convergence rate for a simpler algorithm.
We assume that in line 8, instead of estimating by subsampling, we compute exact weight by going all over the dataset.
The following lemma characterizes the convergence rate of this method:

\begin{lemma} \label{mainthm}
Given $K,T,R>0$, and $\eta$ small. Then the algorithm described above converges to the optimal solution linearly in expectation. 
Moveover, the itaration $\bar{x}_k$ approaches to optimal solution $x^*$ as:
\begin{align*}
   \mathbb{E}[F(\bar{x}_{k})-F(x^*)]\le \rho \mathbb{E}[F(\bar{x}_{k-1})-F(x^*)],
\end{align*}
where $\rho$ defined as 
$$\rho=\frac{1}{\lambda\eta T\big(1-\eta\left(1+\frac{2}{R}\right)\bar{L}\big)}+\frac{2\eta\frac{\bar{L}}{R}}{1-\eta\left(1+\frac{2}{R}\right)\bar{L}}.$$
\end{lemma}

As this more straightforward version of the algorithm requires inefficient gradient computation at the inner loop, we instead proposed estimating it in $\mathbf{ASD-SVRG}$.
Moreover, using the help of Lemma \ref{noisydist} we obtain a similar rate with a slight difference.
The resulting convergence rate changes only by having $\frac{2+2\gamma}{R}$  instead of $\frac{2}{R}$ (for all terms in numerator and denominator).
It is easy to notice that for small enough $\tau$'s we have $\gamma > 2.5 \tau$ satisfied in which gives us the final convergence rate.
\end{proof}

The straightforward application of the theorem gives us the following conclusion. \\

\begin{corollary}
Let $\eta=\frac{\theta}{\bar{L}}$ where $0<\theta<\frac{1}{3}$. 
Then, for any  $\epsilon>0$, after $\mathcal{O}\left(\left(\frac{\bar{L}}{\gamma}n+n\right)\log\left(\frac{1}{\epsilon}\right)\right)$ many
iterations we  obtain an $\epsilon$-approximation solution, (i.e
$\, \mathbb{E}[F(\bar{x}_k)-F(x^*)]\le \epsilon$) .
\end{corollary}

\begin{table*}[!ht]
\setlength{\tabcolsep}{2pt}
\centering
\begin{tabular}{ccc}
\includegraphics[height=4.1cm]{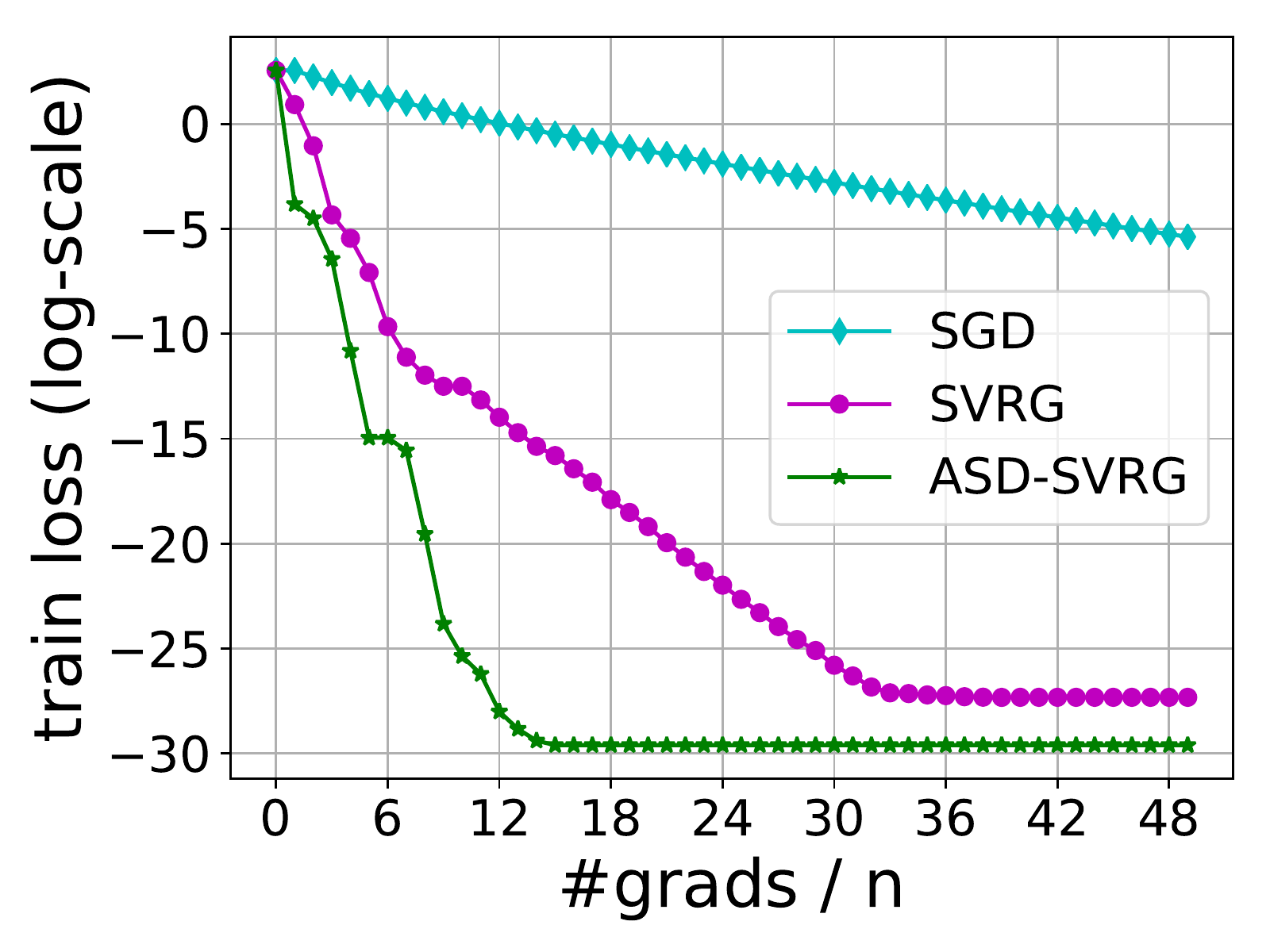} &
\includegraphics[height=4.1cm]{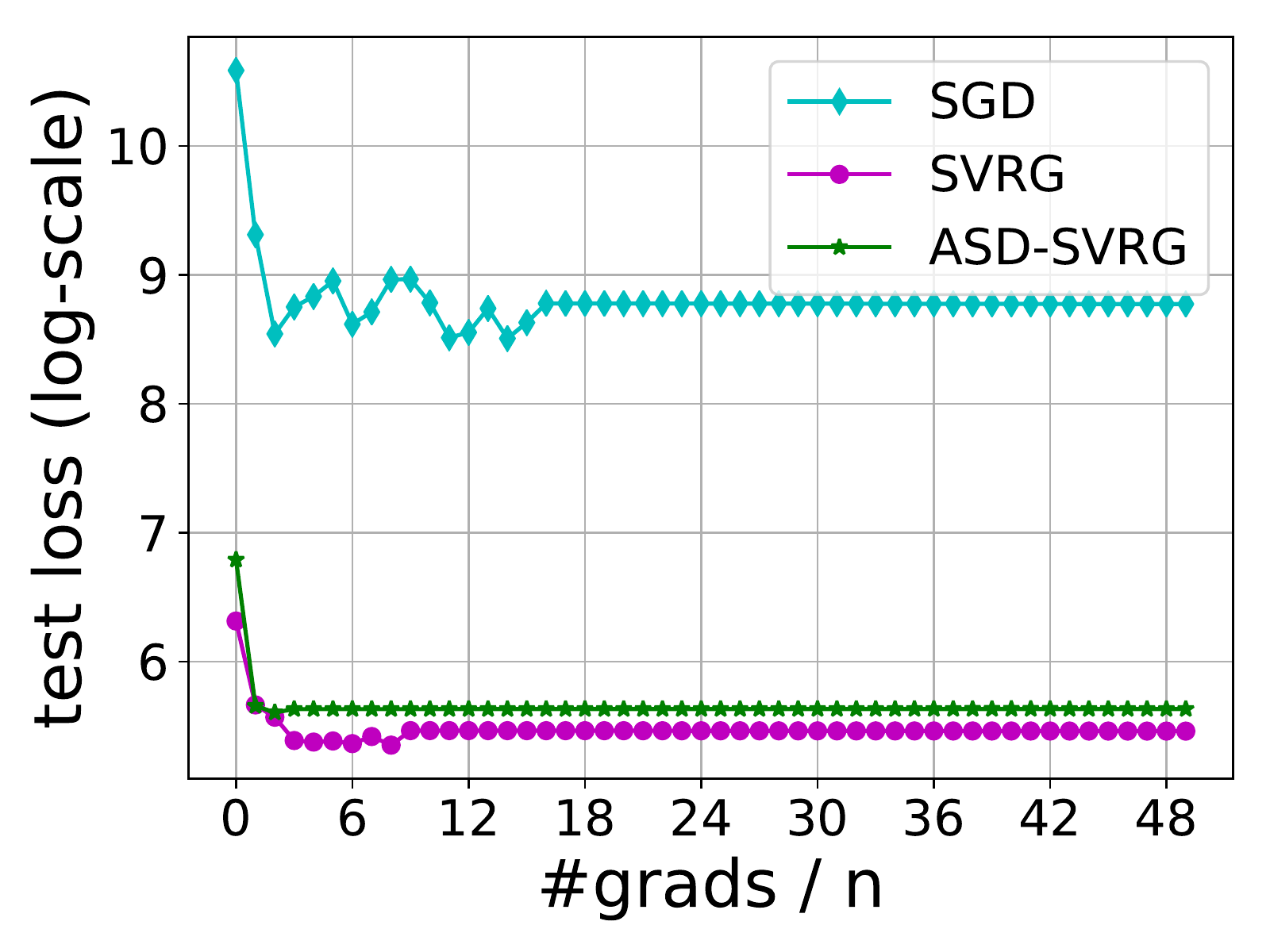} & 
\includegraphics[height=4.1cm]{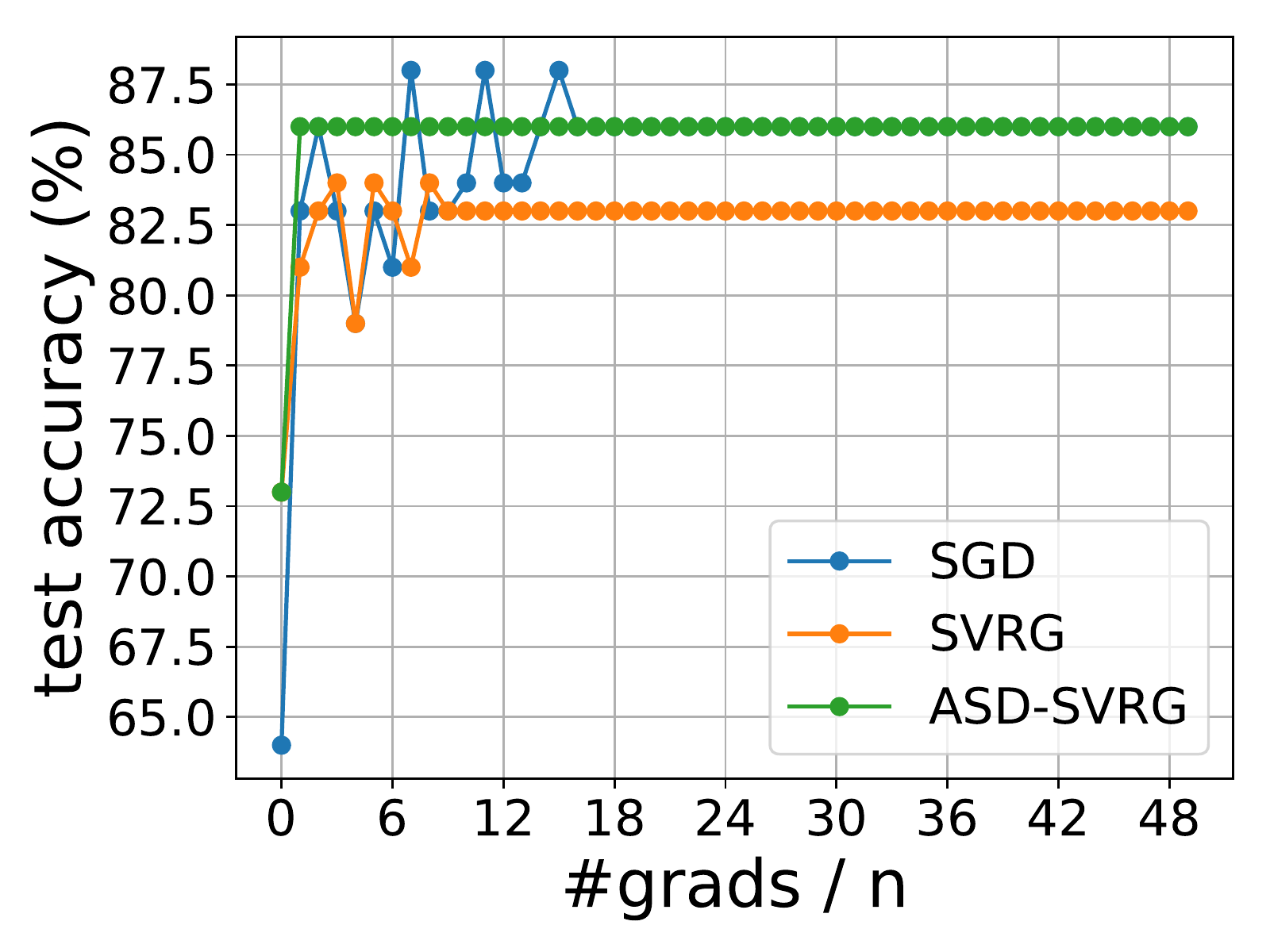} \\
\end{tabular}
\captionof{figure}[]{Best results for each algorithm of SGD, SVRG and ASD-SVRG on our synthetic dataset for logistic regression. ASD-SVRG is able to optimize the train loss much faster than the other two, and although the test loss is a little worse than SVRG, the test accuracy is much higher, proving the advantages of ASD-SVRG. Plus, SVRG is best at learning rate of $7.5e$--$5$ while ASD-SVRG is at $2.5$e--$3$, a lot larger and so is more robust and efficient. 
Left: train loss, center: test loss, right: test accuracy. The losses are in log scale. 
}
\label{fig:best-logistic}
\end{table*}

\begin{table*}[!htb]
\setlength{\tabcolsep}{2pt}
\centering
\begin{tabular}{ccc}
\includegraphics[height=4.1cm]{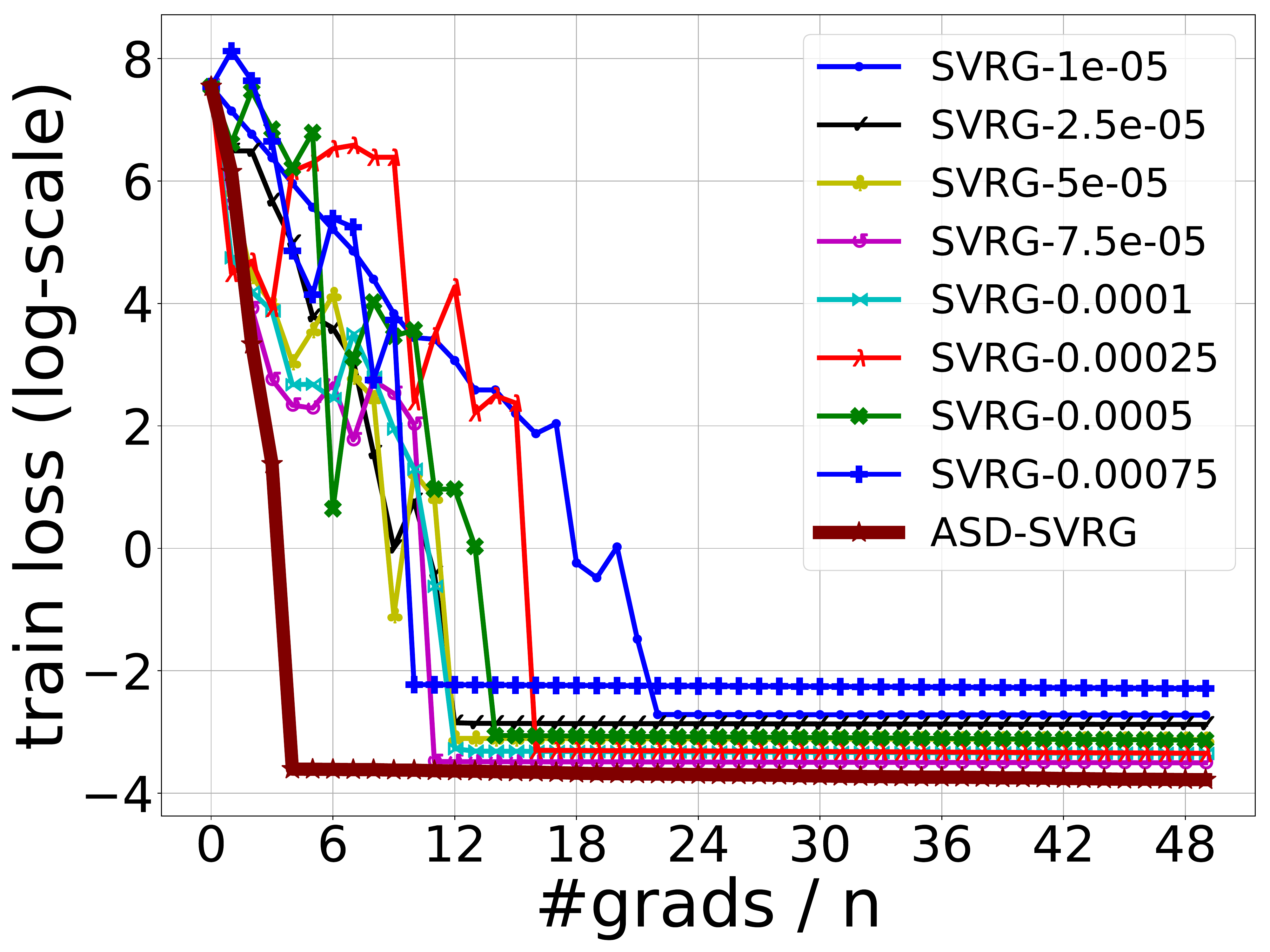} &
\includegraphics[height=4.1cm]{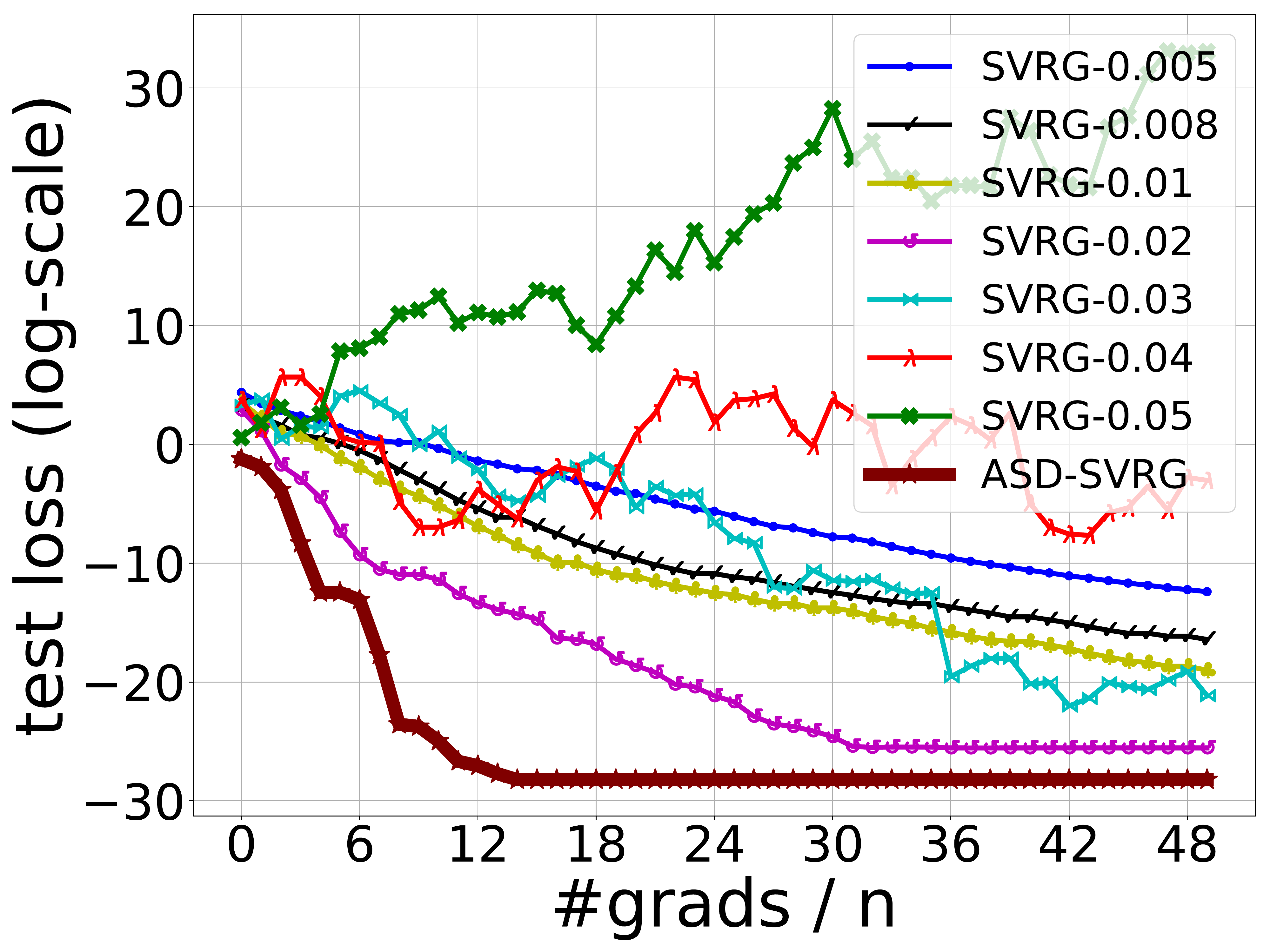} & 
\includegraphics[height=4.1cm]{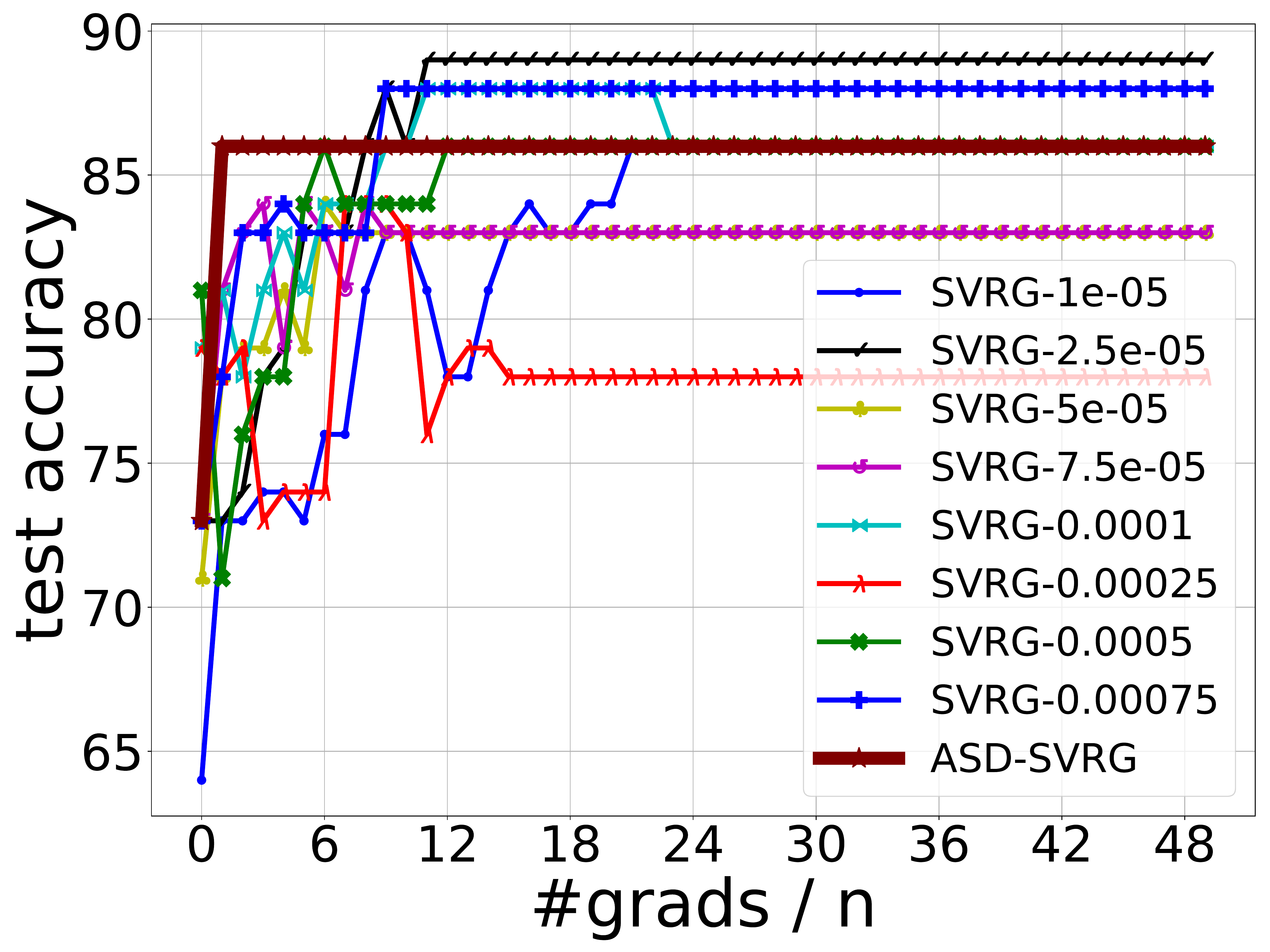}\\
\end{tabular}
\captionof{figure}[]{Ablation study on different learning rates of SVRG compared to ASD-SVRG on the synthetic logistic problem. Except for learning rates, all other setting are the same for both algorithms. 
Train loss (main goal of optimization) and test loss (for generalization) clearly show that ASD-SVRG clearly outperforms other two. For accuracy, although ASD-SVRG is not the best, but those that have higher accuracy are worse in terms of train and test losses. 
Left-to-right: train loss, test loss, test accuracy. The losses are in log scale. }
\label{fig:ablation-logistic}
\end{table*}

\section{Experiments}
To empirically validate our proposed distributed algorithm ASD-SVRG, we compare them to the two baselines which are Distributed SGD and Distributed SVRG, both of which treat all workers uniformly. 

To make the an objective comparison, we initialize the same settings for all of them in that each run has the same number of epochs. For distributed systems perspective, we employ the data parallelization manner, in which the whole data are split into workers. Furthermore, each worker employs the same model architecture, with the same initialization of weights. We describe those experiments in more details in the following sections.

\begin{table}[!ht]
\setlength{\tabcolsep}{0pt}
\centering
\begin{tabular}{cc}
\includegraphics[height=3.1cm]{syn_train_convex_8.pdf} &
\includegraphics[height=3.1cm]{syn_test_convex_8.pdf} \\
\end{tabular}
\captionof{figure}[]{Best results for each algorithm of SGD, SVRG and ASD-SVRG on our synthetic dataset for linear regression. SVRG performs best at learning rate $0.02$ while ASD-SVRG is at $0.2$. Except for SGD The plots show ASD-SVRG is able to bring down the train loss lower and much faster, and hence is more efficient in optimization perspective, while almost has the same performance for test loss.
Left: train loss, right: test loss, both are in log scale.}
\label{fig:best-syn}
\end{table}

\begin{table}[!h]
\setlength{\tabcolsep}{0pt}
\centering
\begin{tabular}{cc}
\includegraphics[height=3.1cm]{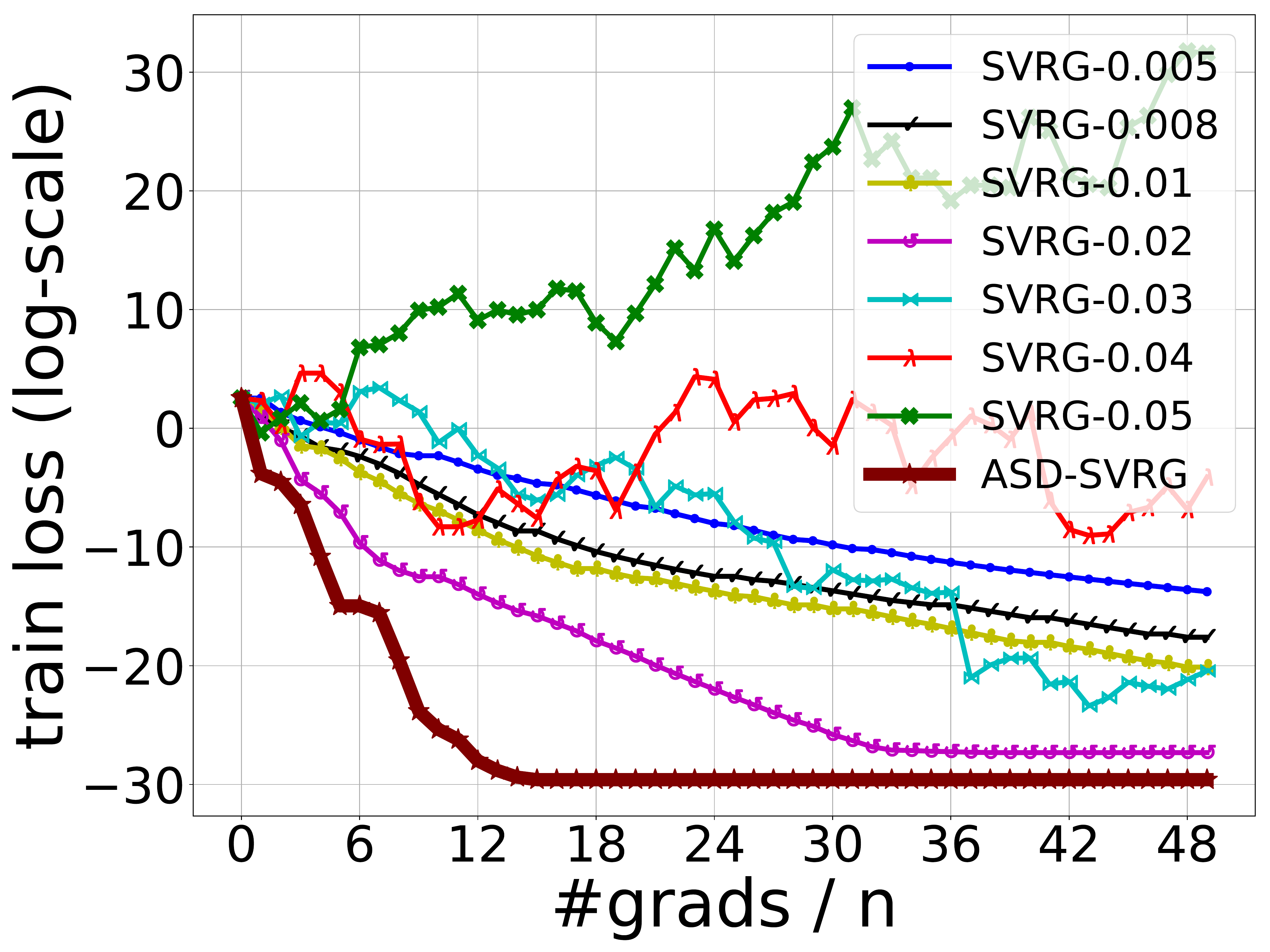} &
\includegraphics[height=3.1cm]{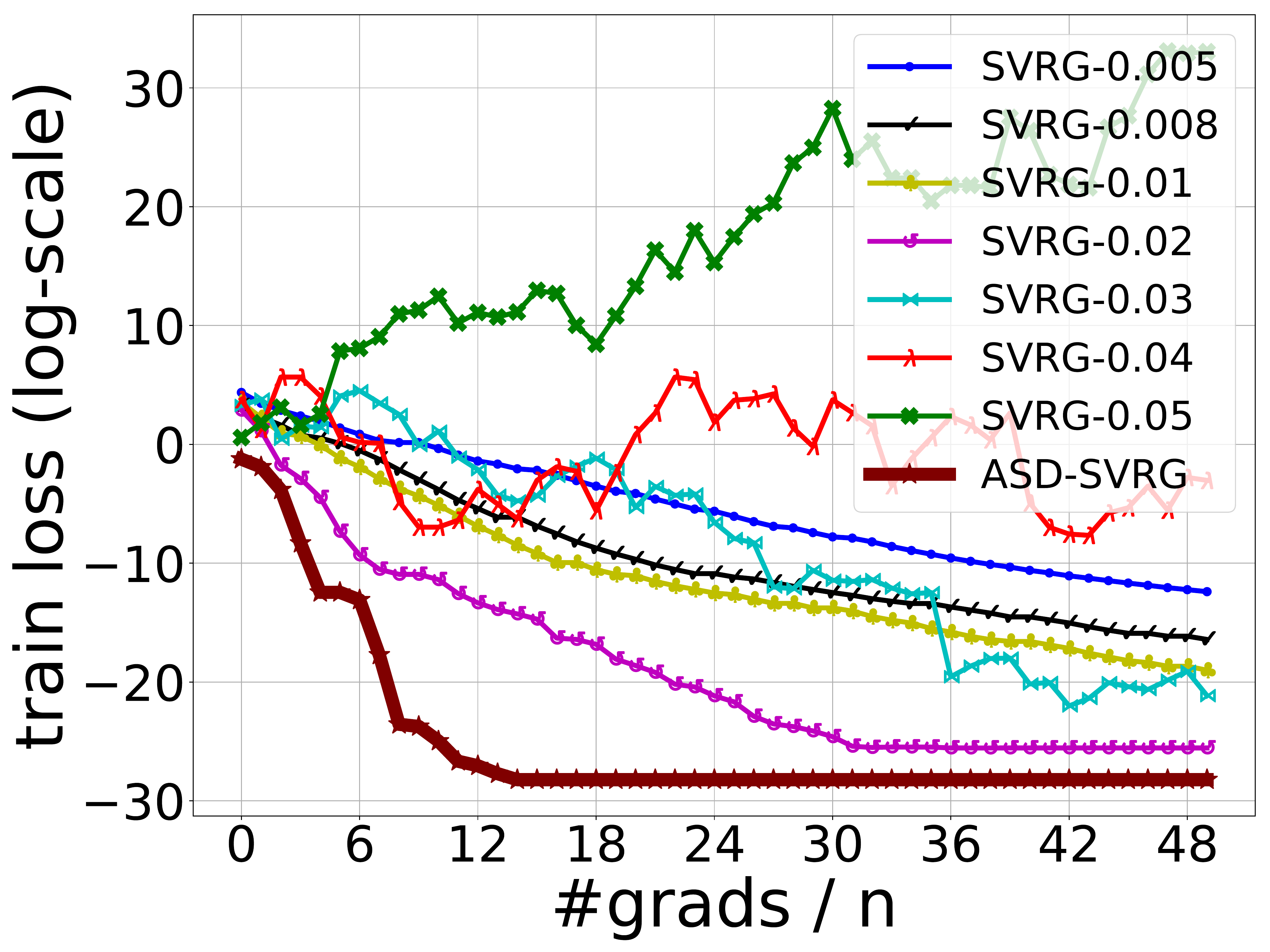} \\
\end{tabular}
\captionof{figure}[]{Ablation study on different learning rates of SVRG compared to our ASD-SVRG on the synthetic regression problem. The results show the advantage of ASD-SVRG on any run of SVRG. ASD-SVRG is less sensitive to the choice of learning rate hyperparameters as compared to SVRG. Left: train loss, right: test loss, both are in log scale.}
\label{fig:ablation-syn}
\end{table}

\subsection{Task and Dataset} 

We design two synthetic datasets for two tasks that has strong convex objective functions: linear regression and logistic regression. For each one, we create increasing Lipschitz constants by worker indices, using an exponential function of those indices. For linear regression, we generate 500 samples of dimension 10 and for logistic regression (in the form of binary classification), we generate 300 samples of dimension 100. 

\paragraph{Experimental Setup} 

To build a convex objective function for both tasks, we build a simple neural network that has only 1 fully-connected layer, which directly transforms the input space to 1 for linear regression, or 2 for logistic regression. This layer is equivalent to a matrix multiplication with an added bias, which is linear in combination. 

For linear regression, we apply mean-square error (MSE) loss. For logistic regression, however, we use log-softmax activation for the logits (for numerical stability), followed by a negative log-likelihood (NLL) loss. In detail, the combination of log-softmax and NLL is equivalent to CrossEntropy loss. 
For optimizers, for both problems, we apply L2 regularization of rate $0.02$ to SGD which is equivalent to weight decaying~\cite{loshchilov2017decoupled}. 
However, for SVRG and ASD-SVRG, we do not apply this regularization mechanism. Finally, to make a fair comparison, we do a grid search of learning rates of each algorithm and compare the best version of each one. Likewise, for each problem, the best performance of each one is yielded by a learning rate different from those of the other two. 

In terms of physical settings, we use a set of 8 paralleled CPUs in a single physical host to run each experiment. And because we implement our code in Pytorch~\cite{paszke2019pytorch}, all the 3 distributed algorithms can be easily adapted to other settings of network architectures (of either convex or non-convex) or distributed configurations such as using parallel GPUs or multiple nodes with many GPUs/CPUs per each. We release our source code at \url{anonymous} for replication of our results.

\paragraph{Results} 
For linear regression, as shown in Figure \ref{fig:best-syn}, ASD-SVRG clearly outperforms others in training and testing: it converges much earlier, and more efficiently, especially in training, which is the main goal in terms of optimization perspective. 
In particular, it also does that with much higher learning rate, which plays an essential role in training speed. 

For logistic regression, Figure \ref{fig:best-logistic} shows a similar behavior in the first two plots. In more detail, our algorithm ASD-SVRG significantly outperforms others for the training set (the main goal), with a learning rate $33$ times larger than SVRG.  We also observe that ASD-SVRG converges faster as compared to SVRG which is consistent with our theoretical analyses. 
Furthermore, ASD-SVRG does not trade generalization for optimization goal. In particular, as shown in the third plot of test accuracy, ASD-SVRG achieves much higher accuracy in prediction ($86\%$ vs $83\%$) while achieving similar test loss compared to SVRG. 

In summary, ASD-SVRG clearly outperforms the baseline methods in both tasks.

\vspace{-1mm}
\paragraph{Ablation Study}
To investigate the difference between ASD-SVRG and its direct counterpart SVRG, we fix the setting of both algorithms and vary the learning rate of SVRG (to the left and right of its best one) to compare its performances with the best setting of ASD-SVRG for both tasks. As shown in Figure \ref{fig:ablation-syn}, ASD-SVRG can easily outperform SVRG in every case for both training and testing for linear regresion. Additionally, if we increase the learning rate towards that of ASD-SVRG (which is many times much larger), SVRG behaves unstably and diverges even at the much lower rates. 

The same observations also happens to logistic regression, as shown in Figure \ref{fig:ablation-logistic}, in both train and test losses. Although some SVRG's rates are better in terms of accuracy (only on a margin of 1\% to 3\%), all of them are worse in terms of both losses. All in all, in terms of optimization perspective, our ablation studies clearly show advantages of ASD-SVRG over SVRG.

\section{Conclusion}
In this paper, 
we have designed and presented a distributed optimization algorithm, namely ASD-SVRG, which assumes no prior knowledge about optimizing functions. Instead, our algorithm is adaptive, in which it samples the most important machines based on data themselves at each step to guide the updates in the optimization process. 
That way, our algorithm is faster converged by redirecting the dependence of convergence rate from maximum to average Lipschitz constants across distributed machines. 
We also provide a statistical categorical distribution decomposition method, which estimates noisy distributions with noiseless versions.
Moreover, we created a novel communication method that effectively minimizes the number of bytes transferred to the parameter server meanwhile having efficient overall run time, both of which are important in practice of distributed algorithms.
For experiments, we implement all algorithms in Pytorch for the ease of adaptation and extension in future, and will also release the code to the public community for results replication. 
We hope that our theoretical results and empirical tools provided in this paper would inspire and help future works in this area.

\bibliography{main}

\begin{thebibliography}{33}
\providecommand{\natexlab}[1]{#1}
\providecommand{\url}[1]{\texttt{#1}}
\expandafter\ifx\csname urlstyle\endcsname\relax
  \providecommand{\doi}[1]{doi: #1}\else
  \providecommand{\doi}{doi: \begingroup \urlstyle{rm}\Url}\fi

\bibitem[Alain et~al.(2015)Alain, Lamb, Sankar, Courville, and
  Bengio]{alain2015variance}
Alain, G., Lamb, A., Sankar, C., Courville, A., and Bengio, Y.
\newblock Variance reduction in sgd by distributed importance sampling.
\newblock \emph{arXiv preprint arXiv:1511.06481}, 2015.

\bibitem[Bardenet(2015)]{concentrationwithoutreplacement}
Bardenet, Rémi;~Maillard, O.-A.
\newblock Concentration inequalities for sampling without replacement.
\newblock In \emph{Bernoulli 21 (2015), no. 3}, pp.\  1361--1385, 2015.

\bibitem[Beck(2017)]{beck2017first}
Beck, A.
\newblock \emph{First-order methods in optimization}, volume~25.
\newblock SIAM, 2017.

\bibitem[Bottou(2010)]{bottou2010large}
Bottou, L.
\newblock Large-scale machine learning with stochastic gradient descent.
\newblock In \emph{Proceedings of COMPSTAT'2010}, pp.\  177--186. Springer,
  2010.

\bibitem[Bottou et~al.(2018)Bottou, Curtis, and
  Nocedal]{bottou2018optimization}
Bottou, L., Curtis, F.~E., and Nocedal, J.
\newblock Optimization methods for large-scale machine learning.
\newblock \emph{Siam Review}, 60\penalty0 (2):\penalty0 223--311, 2018.

\bibitem[Bouchard et~al.(2015)Bouchard, Trouillon, Perez, and
  Gaidon]{bouchard2015online}
Bouchard, G., Trouillon, T., Perez, J., and Gaidon, A.
\newblock Online learning to sample.
\newblock \emph{arXiv preprint arXiv:1506.09016}, 2015.

\bibitem[Chen et~al.(2018)Chen, Giannakis, Sun, and Yin]{chen2018lag}
Chen, T., Giannakis, G., Sun, T., and Yin, W.
\newblock Lag: Lazily aggregated gradient for communication-efficient
  distributed learning.
\newblock In \emph{Advances in Neural Information Processing Systems}, pp.\
  5050--5060, 2018.

\bibitem[Dean \& Ghemawat(2008)Dean and Ghemawat]{dean2008mapreduce}
Dean, J. and Ghemawat, S.
\newblock Mapreduce: simplified data processing on large clusters.
\newblock \emph{Communications of the ACM}, 51\penalty0 (1):\penalty0 107--113,
  2008.

\bibitem[Dean et~al.(2012)Dean, Corrado, Monga, Chen, Devin, Mao, Ranzato,
  Senior, Tucker, Yang, et~al.]{dean2012large}
Dean, J., Corrado, G., Monga, R., Chen, K., Devin, M., Mao, M., Ranzato, M.,
  Senior, A., Tucker, P., Yang, K., et~al.
\newblock Large scale distributed deep networks.
\newblock In \emph{Advances in neural information processing systems}, pp.\
  1223--1231, 2012.

\bibitem[Defazio et~al.(2014)Defazio, Bach, and
  Lacoste-Julien]{defazio2014saga}
Defazio, A., Bach, F., and Lacoste-Julien, S.
\newblock Saga: A fast incremental gradient method with support for
  non-strongly convex composite objectives.
\newblock In \emph{Advances in neural information processing systems}, pp.\
  1646--1654, 2014.

\bibitem[Dekel et~al.(2012)Dekel, Gilad-Bachrach, Shamir, and
  Xiao]{dekel2012optimal}
Dekel, O., Gilad-Bachrach, R., Shamir, O., and Xiao, L.
\newblock Optimal distributed online prediction using mini-batches.
\newblock \emph{Journal of Machine Learning Research}, 13\penalty0
  (Jan):\penalty0 165--202, 2012.

\bibitem[Duchi et~al.(2013)Duchi, Jordan, and McMahan]{duchi2013estimation}
Duchi, J., Jordan, M.~I., and McMahan, B.
\newblock Estimation, optimization, and parallelism when data is sparse.
\newblock In \emph{Advances in Neural Information Processing Systems}, pp.\
  2832--2840, 2013.

\bibitem[Johnson \& Zhang(2013)Johnson and Zhang]{johnson2013accelerating}
Johnson, R. and Zhang, T.
\newblock Accelerating stochastic gradient descent using predictive variance
  reduction.
\newblock In \emph{Advances in neural information processing systems}, pp.\
  315--323, 2013.

\bibitem[Johnson \& Guestrin(2018)Johnson and Guestrin]{johnson2018training}
Johnson, T.~B. and Guestrin, C.
\newblock Training deep models faster with robust, approximate importance
  sampling.
\newblock In \emph{Advances in Neural Information Processing Systems}, pp.\
  7265--7275, 2018.

\bibitem[Katharopoulos \& Fleuret(2017)Katharopoulos and
  Fleuret]{katharopoulos2017biased}
Katharopoulos, A. and Fleuret, F.
\newblock Biased importance sampling for deep neural network training.
\newblock \emph{arXiv preprint arXiv:1706.00043}, 2017.

\bibitem[Katharopoulos \& Fleuret(2018)Katharopoulos and
  Fleuret]{katharopoulos2018not}
Katharopoulos, A. and Fleuret, F.
\newblock Not all samples are created equal: Deep learning with importance
  sampling.
\newblock \emph{arXiv preprint arXiv:1803.00942}, 2018.

\bibitem[Kone{\v{c}}n{\`y} et~al.(2016)Kone{\v{c}}n{\`y}, McMahan, Ramage, and
  Richt{\'a}rik]{konevcny2016federated}
Kone{\v{c}}n{\`y}, J., McMahan, H.~B., Ramage, D., and Richt{\'a}rik, P.
\newblock Federated optimization: Distributed machine learning for on-device
  intelligence.
\newblock \emph{arXiv preprint arXiv:1610.02527}, 2016.

\bibitem[Le~Roux et~al.(2013)Le~Roux, Schmidt, and Bach]{le2013stochastic}
Le~Roux, N., Schmidt, M.~W., and Bach, F.~R.
\newblock A stochastic gradient method with an exponential convergence rate for
  finite training sets.
\newblock \emph{Advances in Neural Information Processing Systems},
  25:\penalty0 3--6, 2013.

\bibitem[Li et~al.(2014)Li, Andersen, Smola, and Yu]{li2014communication}
Li, M., Andersen, D.~G., Smola, A.~J., and Yu, K.
\newblock Communication efficient distributed machine learning with the
  parameter server.
\newblock In \emph{Advances in Neural Information Processing Systems}, pp.\
  19--27, 2014.

\bibitem[Loshchilov \& Hutter(2017)Loshchilov and
  Hutter]{loshchilov2017decoupled}
Loshchilov, I. and Hutter, F.
\newblock Decoupled weight decay regularization.
\newblock \emph{arXiv preprint arXiv:1711.05101}, 2017.

\bibitem[Paszke et~al.(2019)Paszke, Gross, Massa, Lerer, Bradbury, Chanan,
  Killeen, Lin, Gimelshein, Antiga, et~al.]{paszke2019pytorch}
Paszke, A., Gross, S., Massa, F., Lerer, A., Bradbury, J., Chanan, G., Killeen,
  T., Lin, Z., Gimelshein, N., Antiga, L., et~al.
\newblock Pytorch: An imperative style, high-performance deep learning library.
\newblock In \emph{Advances in neural information processing systems}, pp.\
  8026--8037, 2019.

\bibitem[Recht et~al.(2011)Recht, Re, Wright, and Niu]{recht2011hogwild}
Recht, B., Re, C., Wright, S., and Niu, F.
\newblock Hogwild: A lock-free approach to parallelizing stochastic gradient
  descent.
\newblock In \emph{Advances in neural information processing systems}, pp.\
  693--701, 2011.

\bibitem[Reddi et~al.(2015)Reddi, Hefny, Sra, Poczos, and
  Smola]{reddi2015variance}
Reddi, S.~J., Hefny, A., Sra, S., Poczos, B., and Smola, A.~J.
\newblock On variance reduction in stochastic gradient descent and its
  asynchronous variants.
\newblock In \emph{Advances in Neural Information Processing Systems}, pp.\
  2647--2655, 2015.

\bibitem[Reddi et~al.(2016)Reddi, Kone{\v{c}}n{\`y}, Richt{\'a}rik,
  P{\'o}cz{\'o}s, and Smola]{reddi2016aide}
Reddi, S.~J., Kone{\v{c}}n{\`y}, J., Richt{\'a}rik, P., P{\'o}cz{\'o}s, B., and
  Smola, A.
\newblock Aide: Fast and communication efficient distributed optimization.
\newblock \emph{arXiv preprint arXiv:1608.06879}, 2016.

\bibitem[Robbins \& Monro(1951)Robbins and Monro]{robbins1951stochastic}
Robbins, H. and Monro, S.
\newblock A stochastic approximation method.
\newblock \emph{The annals of mathematical statistics}, pp.\  400--407, 1951.

\bibitem[Shamir et~al.(2014)Shamir, Srebro, and Zhang]{shamir2014communication}
Shamir, O., Srebro, N., and Zhang, T.
\newblock Communication-efficient distributed optimization using an approximate
  newton-type method.
\newblock In \emph{International conference on machine learning}, pp.\
  1000--1008, 2014.

\bibitem[Stich et~al.(2017)Stich, Raj, and Jaggi]{stich2017safe}
Stich, S.~U., Raj, A., and Jaggi, M.
\newblock Safe adaptive importance sampling.
\newblock In \emph{Advances in Neural Information Processing Systems}, pp.\
  4381--4391, 2017.

\bibitem[Xiao \& Zhang(2014)Xiao and Zhang]{xiao2014proximal}
Xiao, L. and Zhang, T.
\newblock A proximal stochastic gradient method with progressive variance
  reduction.
\newblock \emph{SIAM Journal on Optimization}, 24\penalty0 (4):\penalty0
  2057--2075, 2014.

\bibitem[Zaharia et~al.(2010)Zaharia, Chowdhury, Franklin, Shenker, and
  Stoica]{zaharia2010spark}
Zaharia, M., Chowdhury, M., Franklin, M.~J., Shenker, S., and Stoica, I.
\newblock Spark: Cluster computing with working sets.
\newblock \emph{HotCloud}, 10\penalty0 (10-10):\penalty0 95, 2010.

\bibitem[Zhang et~al.(2012)Zhang, Wainwright, and
  Duchi]{zhang2012communication}
Zhang, Y., Wainwright, M.~J., and Duchi, J.~C.
\newblock Communication-efficient algorithms for statistical optimization.
\newblock In \emph{Advances in Neural Information Processing Systems}, pp.\
  1502--1510, 2012.

\bibitem[Zhang et~al.(2013)Zhang, Duchi, Jordan, and
  Wainwright]{zhang2013information}
Zhang, Y., Duchi, J., Jordan, M.~I., and Wainwright, M.~J.
\newblock Information-theoretic lower bounds for distributed statistical
  estimation with communication constraints.
\newblock In \emph{Advances in Neural Information Processing Systems}, pp.\
  2328--2336, 2013.

\bibitem[Zhao \& Zhang(2015)Zhao and Zhang]{zhao2015stochastic}
Zhao, P. and Zhang, T.
\newblock Stochastic optimization with importance sampling for regularized loss
  minimization.
\newblock In \emph{international conference on machine learning}, pp.\  1--9,
  2015.

\bibitem[Zinkevich et~al.(2010)Zinkevich, Weimer, Li, and
  Smola]{zinkevich2010parallelized}
Zinkevich, M., Weimer, M., Li, L., and Smola, A.~J.
\newblock Parallelized stochastic gradient descent.
\newblock In \emph{Advances in neural information processing systems}, pp.\
  2595--2603, 2010.

\end{thebibliography}
\bibliographystyle{icml2020}

\appendix

\section{Proof of Lemma \ref{lmm1}}

\begin{lemma*} 
Let $S = \{ a_1, a_2, \ldots , a_N \}$ be a set of vectors that $a_i \in \mathbb{R}^d$,  $a \leq a_i \leq b$ for any $i \in [N]$ and fixed vectors $a,b \in \mathbb{R}^d$.
$\mu$ denotes the average of vectors in S: $\mu = \frac{1}{N} \sum_{i=1}^N a_i$ and
$X_1,X_2,...,X_n$ is the set of size $n$ that uniformly sampled without replacement from $S$. 
Given that $n = \frac{1}{\tau^2} \frac{\|b-a\|^2}{2 \mu^2} \log{\frac{2d}{\epsilon}}$ then the 
following inequality is satisfied with probability at least $1- \epsilon$:
\begin{align*} 
     \left\|\frac{1}{n} \sum_{i=1}^n X_i - \mu \right\|_2 \leq \tau \| \mu \|_2
\end{align*}
\end{lemma*}

\begin{proof}

 We use the following concentration inequality to bound the estimation error \cite{concentrationwithoutreplacement}:\medskip\\
\textit{Lemma :
Let $\chi = \{ a_1, a_2, \ldots , a_N \}$ be a set of real points which satisfies 
$\mu = \frac{1}{N} \sum_{i=1}^N a_i$ and $a \leq a_i \leq b$ for  any $i \in [N]$ and real numbers $a$ and $b$.
Lets draw uniform randomly  $X_1,X_2,...,X_n$ without replacement from the set $\chi$.
Then, with probability higher than $1-\epsilon$ the following inequality satisfied:
\begin{align*}
    \frac{1}{n} \sum_{i=1}^n X_i - \mu \leq (b-a) \sqrt{\frac{\rho_n \log{1/\epsilon}}{2n}}
\end{align*}
where we define
\begin{align*}
     \rho_n = 
    \begin{cases}
        1 -\frac{n-1}{N} & \hspace{0.6cm} \text{if }  \hspace{0.3cm} n\leq N/2 \\
        1-\frac{n}{N} & \hspace{0.6cm} \text{if }  \hspace{0.3cm} n > N/2
    \end{cases}    
\end{align*}
}
using the fact $\rho_n \leq 1$ we conclude:
\begin{align*}
    \frac{1}{n} \sum_{i=1}^n X_i - \mu \leq (b-a) \sqrt{\frac{\rho_n \log{1/\epsilon}}{2n}}  \leq (b-a) \sqrt{\frac{ \log{1/\epsilon}}{2n}} 
\end{align*}
and applying the same inequality to the set $\bar{\chi} = \{-a_1,-a_2,\ldots,-a_N\}$ we conclude with probability at least $1-\epsilon$:
\begin{align*}
   - \frac{1}{n} \sum_{i=1}^n X_i + \mu \geq (b-a) \sqrt{\frac{ \log{1/\epsilon}}{2n}}.
\end{align*}
Using union bound gives us with probability at least $1-2\epsilon$ the following inequality satisfied:
\begin{align*}
    | \frac{1}{n} \sum_{i=1}^n X_i - \mu | \leq (b-a) \sqrt{\frac{ \log{1/\epsilon}}{2n}}
\end{align*}
To extend this inequality to vectors, we assume that $a_i =(a_i^1,a_i^2,\ldots,a_i^d)$ and we denote $\mu^j$ by the average of $j$'th coordinates: $\mu^j = \frac{1}{N} \sum_{i=1}^{N} a_i^j$. $a^j$ and $b^j$ will stand for corresponding upper and lower bounds for $j$'th coordinate. Finally, $X_i^j$ is $j$'th coordinate of $i$-th randomly selected element.
Then, for each $j\in [d]$, the following is satisfied with probability $1-2\epsilon$:
\begin{align*}
    | \frac{1}{n} \sum_{i=1}^n X^j_i - \mu^j | \leq (b^j-a^j) \sqrt{\frac{ \log{1/\epsilon}}{2n}}
\end{align*}
Again using union bound we conclude that with a probability of $1-2d\epsilon$ all of the following inequalities are satisfied:
\begin{align*}
    &| \frac{1}{n} \sum_{i=1}^n X^1_i - \mu^1 |^2 \leq (b^1-a^1)^2 {\frac{ \log{1/\epsilon}}{2n}}\\
    &| \frac{1}{n} \sum_{i=1}^n X^2_i - \mu^2 |^2 \leq (b^2-a^2)^2 {\frac{ \log{1/\epsilon}}{2n}}\\
    &\vdots \\
    &|\frac{1}{n} \sum_{i=1}^n X^d_i - \mu^d |^2 \leq (b^d-a^d)^2 {\frac{ \log{1/\epsilon}}{2n}}\\   
\end{align*}
Summing all of the terms in the left and right side, we conclude with probability $1-2d\epsilon$:
\begin{align*}
    &\| \frac{1}{n} \sum_{i=1}^n X_i - \mu \|_2^2 \leq \|b-a\|_2^2 {\frac{ \log{1/\epsilon}}{2n}}
\end{align*}
satisfied. Hence plugging $n = \frac{1}{\tau^2} \frac{ \| b-a \|^2_2 }{2 \mu^2}  \log{1/\epsilon} $ guarantees 
\begin{align*}
    &\| \frac{1}{n} \sum_{i=1}^n X_i - \mu \|_2^2 \leq \tau^2 \| \mu \|^2 
\end{align*}
with probability $1-2d\epsilon$. Therefore, assigning $\epsilon \leftarrow  \frac{\epsilon}{2d}$ and taking square root of each side above implies  having $n = \frac{1}{\tau^2} \frac{ \| b-a \|^2_2 }{2 \mu^2}  \log{2d/\epsilon} $ guarantees the following inequality with probability $1-\epsilon$
\begin{align*}
    &\| \frac{1}{n} \sum_{i=1}^n X_i - \mu \|_2 \leq \tau \| \mu \| 
\end{align*}

\end{proof}

\section{Proof of Lemma \ref{noisydist}}

Similar to the main algorithm, here we denote the $\ell_2$ norm of averages of gradients of machine $m$ by $w_m$. (in the lemma above it corresponds to $\mu$) and its estimation by $\widetilde{w}_m$. Then from the lemma \ref{lmm1} we have $|\widetilde{w}_m -w_m | \leq \tau w_m$ with probability $1-\epsilon$. Hence we can write $\widetilde{w}_m = w_m + \delta$ where $\delta \leq \tau w_m$ with probability $1-\epsilon$.
Lemma \ref{noisydist} gives an interesting property of noisy categorical distributions:

\begin{lemma*} 
Let $\mathcal{P}$ be a categorical distribution with weights $(w_1,w_2,...,w_M)$ and
$\widetilde{\mathcal{P}}$ be perturbed distribution of $\mathcal{P}$ with modified weights as $(w_1+\delta_1,w_2+\delta_2,...,w_M+\delta_M)$. 
Then, if  $\mathcal{Q}$ is another categorical distribution that has  weights 
$(\delta_1 - w_1 \mathrm{min}(\frac{\delta_i}{w_i}), 
\delta_2 - w_2 \mathrm{min}(\frac{\delta_i}{w_i}),  \ldots ,
\delta_M - w_M \mathrm{min}(\frac{\delta_i}{w_i}))$ and 
\begin{equation*}
\gamma = 1- \min_{1 \leq i \leq M}  \frac{ \frac{w_i+\delta_i}{ w_1+\delta_1 + w_2+\delta_2 + \ldots + w_M+\delta_M
} }{ \frac{w_i}{ w_1+w_2+ \ldots + w_M }  }
\end{equation*}
Then, we can decompose $\widetilde{\mathcal{P}}$ to the combination of $\mathcal{P}$ and $\mathcal{Q}$ as following:
\begin{align*}
    \Psi = 
    \begin{cases}
        \text{sample with respect to } \mathcal{P} & \text{with probability } 1- \gamma \\
        \text{sample with respect to } \mathcal{Q} & \text{with probability } \gamma
    \end{cases}
\end{align*}
Moreover,  $\gamma$ is the smallest number that enables decomposing $\widetilde{P}$ to $\mathcal{P}$ and some other categorical distribution. 
\end{lemma*}
\begin{proof}
It is straightforward to notice that $\mathcal{Q}$ is well-defined as:
\begin{align*}
    \delta_j - w_j min(\frac{\delta_i}{w_i}) \geq \delta_j - w_j \frac{\delta_j}{w_j}=0.
\end{align*}
Considering the fact that for $i_0=\mathrm{argmin}_{1\leq i \leq M}( \frac{\delta_i}{w_i} )$ the inequality  
$\delta_i - w_i min(\frac{\delta_i}{w_i})\geq 0$ is tight, 
then $\gamma$ is the smallest number that $\widetilde{P}$ can be decomposed into $\mathcal{P}$ and some other distribution.

Then all we need to do is to show that the probability of selection of category$-j$ of proposed method is equal to probability of category$-j$ for $\widetilde{\mathcal{P}}$ which is $\frac{w_j+\delta_j}{w_1+\delta_1+w_2+\delta_2+\ldots+w_M+\delta_M}$.
Lets find the probability- $\mathbb{P}_{\Psi}(j)$  of category$-j$ for distribution $\Psi$:
\begin{align} \label{noisysum}
&\mathbb{P}_{\Psi}(j)   =  (1-\gamma) \frac{w_j}{w_1+w_2+\ldots+w_M}  \\
&\hspace{5mm}+\gamma \frac{\delta_j -w_j min(\frac{\delta_i}{w_i}) }{
\delta_1 -w_1 min(\frac{\delta_i}{w_i}) + \ldots + \delta_M -w_M min(\frac{\delta_i}{w_i})} \nonumber
\end{align}
Let's do a detailed analysis of each of these summands. We start with the left summand first.
\begin{align*}
(1-\gamma)  \frac{w_j}{w_1+w_2+\ldots+w_M} \hspace{22mm}  \\
 = \frac{w_j}{w_1+\ldots+w_M}
\min_{1 \leq i \leq M}  \frac{ \frac{w_i+\delta_i}{ w_1+\delta_1 + w_2+\delta_2 + \ldots + w_M+\delta_M
} }{ \frac{w_i}{ w_1+w_2+ \ldots + w_M }  } \hspace{3mm} \\
= \frac{w_j}{w_1+\delta_1+\ldots+w_M+\delta_M}
\min_{1 \leq i \leq M} \frac{w_i+\delta_i}{w_i}  \hspace{13mm}\\
= \frac{w_j}{w_1+\delta_1+\ldots+w_M+\delta_M} (1+\min_{1 \leq i \leq M} \frac{\delta_i}{w_i}) \hspace{12mm}
\end{align*}
Now, we focus on understanding the right summand better. First we focus on the value of $\gamma$:

\begin{align*}
  \gamma = 1- \min_{1 \leq i \leq M}  \frac{ \frac{w_i+\delta_i}{ w_1+\delta_1 + w_2+\delta_2 + \ldots + w_M+\delta_M} }{ \frac{w_i}{ w_1+w_2+ \ldots + w_M }  } \hspace{20mm} \\[1.2ex]
= 
1- \frac{w_1+...+w_M}{w_1+\delta_1+\ldots+w_M+\delta_M} \min_{1\leq i \leq M} \frac{w_i+\delta_i}{w_i} \hspace{7mm}      \medskip\\[1.2ex]
= \frac{w_1+\delta_1 + \ldots + w_M +\delta_M}{w_1+\delta_1 + \ldots+w_M+\delta_M} \hspace{35mm} \medskip\\[1.2ex]
- \frac{ \big( w_1+...+w_M \big)  \min_{1\leq i \leq M} \frac{w_i+\delta_i}{w_i} }{w_1+\delta_1+\ldots+w_M+\delta_M}  \hspace{21mm}
\medskip\\[1.2ex]
=\frac{\delta_1+\ldots+\delta_M - \min_{1 \leq i \leq M} \frac{\delta_i}{w_i}(w_1+\ldots + w_M) }
{w_1+\delta_1+\ldots + w_M+\delta_M} \hspace{4mm}
\end{align*}
Therefore, the right summand in (\ref{noisysum}) above will simply be equal to:
\begin{align*}
\gamma \times  \frac{\delta_j -w_j min(\frac{\delta_i}{w_i}) }{
\delta_1 -w_1 min(\frac{\delta_i}{w_i}) + \ldots + \delta_M -w_M min(\frac{\delta_i}{w_i})}  \hspace{5mm} \\[1.2ex]
=\frac{\delta_1+\ldots+\delta_M - \min_{1 \leq i \leq M} \frac{\delta_i}{w_i}(w_1+\ldots + w_M) } {w_1+\delta_1+\ldots + w_M+\delta_M}  \\[1.2ex]
\times \frac{\delta_j -w_j min(\frac{\delta_i}{w_i}) }{
\delta_1+\ldots+\delta_M - \min_{1 \leq i \leq M} \frac{\delta_i}{w_i}(w_1+\ldots + w_M)} \\[1.2ex]
= \frac{\delta_j-w_j\min(\frac{\delta_i}{w_i})}{w_1+\delta_1+\ldots+w_M +\delta_M} \hspace{30mm}.
\end{align*}
Finally, putting these summands together concludes:
\begin{align*}
\frac{w_j (1+\min_{1 \leq i \leq M}\frac{\delta_i}{w_i})}{w_1+\delta_1+\ldots+w_M+\delta_M} +
\frac{\delta_j-w_j\min(\frac{\delta_i}{w_i})}{w_1+\delta_1+\ldots+w_M +\delta_M}= \\
\frac{w_j + \delta_j }{w_1+\delta_1+\ldots+w_M +\delta_M} \hspace{25mm} .
\end{align*}
Which approves that sampling with respect to $\Psi$ is equivalent to sampling with respect to $\widetilde{\mathcal{P}}$.
\end{proof}
Lemma 2 gives an interesting property of $\gamma$, which is an equivalent way telling that when noise is small (i.e. $\gamma$ is close to zero) noisy distribution is behaving almost the same as noiseless which is intuitive.

After algebraic manipulations we notice that 
\begin{align*}
 \frac{1}{1-\gamma}  =  \max_{1 \leq i \leq M}  \frac{
 \frac{w_i}{ w_1+w_2+ \ldots + w_M }  }
 { \frac{w_i+\delta_i}{ w_1+\delta_1 + w_2+\delta_2 + \ldots + w_M+\delta_M} } =
 \max_{1 \leq i \leq M}  \frac{p_i }{\widetilde{p}_i } 
\end{align*}
where $p_i$ stands for the probability of $i$-th   category.
In the following lemma we discuss under bounded heavy noise we can still bound $\gamma$ (i.e. $\frac{1}{1-\gamma}$) which later will be used to show we still have robust convergence rate.
bounding $\tau$ in the \ref{lmm1} gives a bound  $\max_{1 \leq i \leq M}  \frac{p_i }{\widetilde{p}_i } $ which later this bound will be used to prove final convergence rate.

\begin{lemma} \label{oneminus}
 Selecting $\tau = \frac{1}{3}$ above for all machines, will guarantee that 
 $\frac{1 }{1-\gamma} \leq 2$.
\end{lemma}
\begin{proof}
Having $\tau =\frac{1}{3}$  implies that $ | \delta_m | \leq \frac{1}{3} w_m$ for any $m \in [M]$.
Hence,  $w_m +\delta_m \geq \frac{2}{3} w_m$ get satisfied for all $m\in [M]$. 
Using the same inequality $|\delta_m | \leq \frac{1}{3} w_m$ we conclude also that $w_m+\delta_m \leq \frac{4}{3} w_m $ for any $m\in [M]$.
Then,  
$ \sum\limits_{m=1}^{M} (w_m+\delta_m) \leq \frac{4}{3} \sum\limits_{m=1}^M w_m $  get satisfied, in which equivalent to $$\frac{w_1+w_2+\ldots+w_M}{w_1+\delta_1+\ldots w_M+\delta_M}\geq \frac{3}{4}.$$
Returning back to the definition of $\gamma$:
\begin{align*}
\frac{1}{ 1-\gamma }=  \max_{1 \leq i \leq M}  \frac{ 
\frac{w_i}{ w_1+w_2+ \ldots + w_M }  }
{  \frac{w_i+\delta_i}{ w_1+\delta_1 + w_2+\delta_2 + \ldots + w_M+\delta_M}  } 
\hspace{20mm} \\[1.2ex]
= \max_{1 \leq i \leq m}   \frac{w_i}{w_i+\delta_i} \times 
\frac{w_1+\delta_1+\ldots+w_M+\delta_M}{w_1+w_2+\ldots+w_M} \hspace{2mm} \\[1.2ex]
\leq \frac{3}{2} \times \frac{4}{3}  = 2  \hspace{52mm}
\end{align*}
\end{proof}

\section{Proof of Lemma \ref{communication}}

\begin{lemma*}\label{lm1}
The sampling technique \hyperlink{pc}{$\mathbf{PC}$} samples $R$ many workers with replacement using just $\mathcal{O}(M)$ many worker to worker communication for any $R$. 
Furthermore, sampling process ends in total time of $\mathcal{O}(R\log{M})$.
\end{lemma*}

\begin{proof}
In the first part of the proof, we show that the cost of communication is independent of $R$, i.e., the algorithm does $\mathcal{O}(M)$ worker-to-worker communication.
Then, we show that run-time is $\mathcal{O}(R \log{M})$. 
Finally, we will prove that for a machine $i$, the probability of it getting sampled is $\frac{w_i}{w_1+w_2+\ldots+w_M}$. \\\\
\textit{Communication:}
The first step of communication happens in the first line of the algorithm.
Machines $\{1,2,\ldots, R-1\}$ sends 2 scalars to machine $R$, machines  $\{R+1,R+2,\ldots, 2R-1\}$ sends to machine $R$ and so on so forth. 
All of the machines send two scalars except ones that have an index of multiples of $R$. 
Therefore, overall, there has been $2 (M-\frac{M}{R})$ scalar transfer.\medskip\\
Next, we focus on the communication that happens in the loop of line 3. 
Notice that, at each transfer here, instead of sending two scalars, we are sending $R+1$ scalars where $R$ of them is sampled indices, and one is additional weight. 
Moreover, in the first iteration of the loop, there is $\frac{M}{2R}$ transfer, in the second iteration, $\frac{M}{4R}$,  and it continues by decreasing twice after each iteration.
Therefore, the number of scalars sent at each iteration is:
\begin{itemize}
\item  $h = 1$ : \hspace{10mm}  $(R+1)\frac{M}{2R} = \frac{M}{2}+\frac{M}{2R}$ scalars
\item  $h = 2$ : \hspace{10mm} $(R+1)\frac{M}{4R} = \frac{M}{4}+\frac{M}{4R}$ scalars
    
 \hspace{10mm}   \vdots
    
    \item $h = \log{\frac{M}{R}}$: \hspace{5mm} $(R+1)\times 1$ scalars
\end{itemize}
In total, there are: 
\begin{align*}
\frac{M}{2}+\frac{M}{2R}+ \frac{M}{4}+\frac{M}{4R}+ ... + \frac{M}{\frac{M}{R}}+1 \hspace{25mm}\\[1.1ex]
= M \big( \frac{1}{2} +\frac{1}{4}+\ldots \big) + \frac{M}{R}\big(\frac{1}{2} + \frac{1}{4}+ \ldots \big) \hspace{10mm}\\[1.1ex]
\leq M + \frac{M}{R} \hspace{53mm}
\end{align*}
Therefore, combining this number with the number $2(M-\frac{M}{R})$, we conclude that there are  at most $2M - 2\frac{M}{R} + M + \frac{M}{R} = 3M -\frac{M}{R}$ transfers (i.e. $\mathcal{O}(M)$).\\\\
\textit{Running time:}
Here we analyse the running time of the communication algorithm.
Line 1 has a time cost of $\mathcal{O}(R)$ due to each of receivers receiving $R$ many 2-tuples (in parallel).
Each iteration of the loop starting in line 3, has a time cost of $\mathcal{O}(R)$ as each sender sends $R+1$ sized tuples and receiver receiving them.
Considering the fact that the number of iterations is $\log{\frac{M}{R}}$ then total amount of time is bounded by $\mathcal{O}(R \log{\frac{M}{R}})$. \medskip\\
\textit{Correctness of Sampling: }
Lets denote by $(s^M_1,...,s^M_R)$ indices that has been selected after sampling process.
Here, we show that for machine 1, selection in each of these indices has a probability of $\frac{w_1}{w_1+w_2+\ldots+w_m}$, which is precisely the requirement for sampling with a replacement for this machine.
Moreover, the proof here can be applied to any index to get the same conclusion.\medskip\\
For a given index $j\in [R]$, the final sampled index $s_j^M$ being equal to 1, it is necessary it gets selected in each local sampling.\medskip\\
The first time this involving happens is in line 1, and here for any index selection of machine 1 has the probability of $\frac{w_1}{w_1+w_2+\ldots+w_R}$ and so for index $j$ as well.
Given that machine 1 selected in index $j$ in line 1, then the probability of it getting selected for the same index in the first iteration of loop of line 3 is 
\begin{align*}
 \frac{w_1+w_2+\ldots+w_R}{(w_1+w_2+\ldots+w_R)+(w_{R+1}+w_{R+2}+\ldots+w_{2R})}   
\end{align*}
Moreover, given that machine 1 selected in this step as well, the probability of it getting selected in the second iteration of index $j$ is:
\begin{align*}
 \frac{w_1+w_2+\ldots+w_{2R}}{(w_1+w_2+\ldots+w_{2R})+(w_{2R+1}+w_{2R+2}+\ldots+w_{4R})}.
\end{align*}
Moreover, we can generalize this argument, and to conclude that for the index $j$, the probability of selection of machine 1 is:

$ \mathbf{P}(\text{Machine 1 Selected at the end})  $

\hspace{2mm} $= \mathbf{P}(\text{Machine 1 Selected at step 1 })$ 

\hspace{5mm}$ \times \mathbf{P}(\text{Machine 1 Selected at step 2 } | \text{it selected in step 1 }) $\\\\
 \vdots

\hspace{5mm}  $  \times \mathbf{P}(\text{Machine 1 Selected at step $\log{\frac{M}{R}}$ } | \text{it selected} $

\hspace{54mm}  $\text{in step}  \log{\frac{M}{R}}-1 )$
  $$  = \frac{w_1}{w_1+w_2+\ldots+w_R} \hspace{40mm} $$
 $$\times \frac{w_1+\ldots + w_R}{ (w_1+\ldots+w_R)+(w_{R+1}+\ldots+w_{2R})} \hspace{8mm} $$
 $$\times \frac{w_1+\ldots + w_{M/2}}{ (w_1+\ldots+w_{M/2})+(w_{M/2+1}+\ldots+w_{M})}$$
 $$   = \frac{w_1}{w_1+w_2+\ldots + w_{M}} \hspace{38mm}$$

as desired.\\\\
\textit{Optimality:} The proposed communication method is optimal with respect to the number of communications. 
Each machine should be at least one of the sending/receiving processes. 
Considering that there are $M$ many machines, at least $\mathcal{O}(M)$ many communication should happen.
Therefore, our method is optimal with respect to communication.\\
\end{proof} 

\medskip


\subsection{Alternative Optimal Parallel Communication Method}
As we mention in the proof of the lemma \ref{communication}, the algorithm \hyperlink{pc}{$\mathbf{PC}$} is optimal with respect to the communication.
In this section, we provide an alternative algorithm, which keeps the communication cost the same and optimize the runtime further to $\mathcal{O}(R+\log{M})$. 
The idea of the sampling strategy is similar, with the only difference that instead of sending $R$ many indices from a machine to another, we parallelize this process and send just one index.

\begin{algorithm} \label{algo3}
\caption{ \hypertarget{ocm}{Optimal} Communication }

{\textbf{Input:}} weights $\{{w}_1, {w}_2,\ldots {w}_M \}$ and group size $R$

\begin{algorithmic}[1]
\STATE  \textbf{In parallel:} Machine $m \hspace{-.5mm}\in\hspace{-.5mm}[M]$ sends $(m,w_m)$ to machine $\lceil \frac{m}{R} \rceil R$
\STATE \textbf{In parallel:} Machine $m \in [\frac{M}{R}]R $  samples $R$ indices with replacement from the interval $[m-R+1,m]$ with respect to the weights $\{w_{m-R+1},\ldots,w_{m}\}$ (lets represent them $i_1^m,i_2^m,\ldots,i_R^m$)  and set $w_m = \sum\limits_{j=m-R+1}^{m}w_{j}$. 
Then send $(i_j^m,w_m)$ to machine $m-R+j$ for all $j\in[R]$.
\STATE \textbf{In parallel:} Machine $m \in [\frac{M}{R}]R$ sends $(i_j^m,w_m)$ to machine $m-R+j$ for all $j\in[R]$.
Then, the machine $m-R+j$ sets $i^{m-R+j}=i_j^m$ and $w_{m-R+j} = w_m$
\FOR{$h \in [\log{\frac{M}{R}}] $ }
\FOR{$ u \in [R] \textbf{ in parallel} $ }  
\STATE  For $m \in [\frac{M}{2^{h}R }]2^{h}R-R+u$  denote $s_m = m - 2^{h-1}R$.
 \STATE  \textbf{In parallel: } Machine $s_m$ sends $(i^{s_m},w_{s_m})$ to  machine $m$.
 \STATE  \textbf{In parallel: } Machine $m$ samples from $\{ i^{s_m}, i^{m}\}$ with weights $\{w^{s_m}, w^{m} \} $ and assigns result to $i^{m}$ and set $w^{m} \leftarrow w^{m}+w^{s_m}$
\ENDFOR
\ENDFOR
\end{algorithmic}
{\textbf{ Output:}}  $\{i^{M-R+1},i^{M-R+2},\ldots, i^{M} \}$
\end{algorithm}

 The first step here is the same as \hyperlink{pc}{$\mathbf{PC}$}. 
However, right after the sampling process happens in line 2, instead of sending $R$-many indices to each other, sampled indices distributed among previous $R-1$ machines.
In particular, machine R sends $i^R_1$ to machine 1,$i^R_2$ to machine 2, and so on so forth, $i^R_{R-1}$ to machine $R-1$ and keeps $i^R_{R}$ for itself.
Moreover, each of receiver machine also updates their weight to cumulative initial weight.
Then, after each machine $1,2,\ldots,R$ holds an index that sampled according to their relative weights, and each of them has the value of the sum of their weights. 
A similar case happens for machines $R+1,R+2\ldots, 2R$, and all next machines of groups of size $R$.

In the rest of the algorithm, we have $R$ many parallel processes based on mod $R$.
Machines $\{1,R+1,2R+1,\ldots,M-R+1\}$ runs a parallel sampling process to select one index (similar to $\mathbf{PC}$ but just sampling 1 entry). 
The same process happens in the set $\{2,R+2,2R+2,\ldots,M-R+2\}$ and in the set $\{3,R+3,2R+3,\ldots,M-R+3\}$ and etc.
Therefore, at the end, each of the machines $M-R+1,M-R+2,\ldots, M-1, M$ holds a selected index.

Moreover, the first stage of the algorithm has a time complexity of $\mathcal{O}(R)$, and the second stage has a time complexity of $\mathcal{O}(\log{\frac{M}{R}})$.
Hence, cumulative complexity is $\mathcal{O}(R+\log{M})$ and one can notice that communication complexity is still $\mathcal{O}(M)$.

An interesting observation here it selects $R$ many indices using just  $\mathcal{O}(R)$ running time (when $R>\log{M}$).
This time cost is the same with uniform sampling SVRG selecting $R$ many indices in $R$ iterations,
however, in this method, we are selecting the most informative machines, which makes this time much more efficient.

\section{Convergence Rates}

Before moving to the proof of the theorem \ref{mainthm}, we provide background theory, which gives an idea of how the final convergence rate is coming. 
In what follows, we prove the lemma \ref{lmbck} in the next section, and we extend it to the theorem \ref{mainthm} afterward.

\subsection{Precise weights at Each Step}

Here, we analyze a slightly modified version of the main  algorithm.
We assume at line 8, instead of estimation of each weight by subsampling,  we are computing precise weight by going over all of the data points. Then the convergence rate would be characterized in the following lemma:

\begin{lemma*}  \label{lmbck}
Given $K,T,R>0$, and $\eta$ small. Then the algorithm described above converges to the optimal solution linearly in expectation. 
Moveover, the itaration $\bar{x}_k$ approaches to optimal solution $x^*$ as:
\begin{align*}
   \mathbb{E}[F(\bar{x}_{k})-F(x^*)]\le \rho \mathbb{E}[F(\bar{x}_{k-1})-F(x^*)],
\end{align*}
where $\rho$ defined as 
$$\rho=\frac{1}{\lambda\eta T\big(1-\eta\left(1+\frac{2}{R}\right)\bar{L}\big)}+\frac{2\eta\frac{\bar{L}}{R}}{1-\eta\left(1+\frac{2}{R}\right)\bar{L}}.$$ \smallskip 
\end{lemma*}

\begin{proof} 
We use the fact that for any $m=1,\ldots,M$
\begin{align*}
\|\nabla F_m(x)-\nabla F_m(x^*)\|^2_2\le  2L_m \Big[F_m(x)-F_m(x^*) \\
-\nabla F_m(x^*)^T(x-x^*) \Big].   
\end{align*}
This fact can be found in \cite{beck2017first}.
Assume that we are at iteration $t$ in epoch $k$, and we denote the selected (with replacement) machines by $m_1,m_2 \ldots m_R$. Recall that  
\begin{align*}
  v_t=\frac{1}{R}\sum_{r=1}^R\left(\frac{\nabla
F_{m^r}(x_{t-1})}{Mp^{k,t}_{m^r}}-\frac{\nabla F_{m^r}(\bar{x}_{k-1})}{M p^{k,t}_{m^r}}\right)\\[1.2ex]
+\nabla F(\bar{x}_{k-1})  
\end{align*}
Since $\mathbb{E}v_t= \nabla F(x_{t-1})$, we have that 
\begin{align*}
  \mathbb{E}\|v_t-\nabla F(x_{t-1})\|^2=\mathbb{E}\|v_t\|^2-\|\nabla F(x_{t-1})\|^2.  
\end{align*}

Thus, we deduce that 
\begin{align*}
  \mathbb{E}\|v_t\|^2=\mathbb{E}\|v_t-\nabla F(x_{t-1})\|^2+\|\nabla F(x_{t-1})\|^2.  
\end{align*}

In addition, since  
\begin{align*}
 v_t-\nabla F(x_{t-1})=
    \frac{1}{R}\sum_{r=1}^R\left(\frac{\nabla
F_{m^r}(x_{t-1})}{Mp^{k,t}_{m^r}}-\frac{\nabla F_{m^r}(\bar{x}_{k-1})}{M p^{k,t}_{m^r}}\right)\\
-\Big(\nabla F(x_{t-1})-\nabla F(\bar{x}_{k-1})\Big)   
\end{align*}

and 
\begin{align*}
   &\mathbb{E}\left[ \frac{\nabla
F_{m^r}(x_{t-1})}{Mp^{k,t}_{m^r}}-\frac{\nabla F_{m^r}(\bar{x}_{k-1})}{M p^{k,t}_{m^r}}\right]=\nabla F(x_{t-1})-\nabla F(\bar{x}_{k-1}) 
\end{align*}

for all $r=1,\ldots,R$. As the sampling is with replacement, it is independent of the index $r\in [R]$ so we denote $p^{k,t}_{m^r}$ by $p^{k,t}_m$.
Then, after algebraic manipulations, we conclude that
\begin{align*}
 \mathbb{E}\|v_t-\nabla F(x_{t-1})\|^2 =
\frac{1}{R}\mathbb{E}\left\|\frac{\nabla F_{m}(x_{t-1})}{Mp^{k,t}_{m}}-\frac{\nabla F_{m}(\bar{x}_{k-1})}{Mp^{k,t}_{m}}\right\|^2  \\
- \frac{1}{R} \left\| \nabla F(x_{t-1})-\nabla F(\bar{x}_{k-1}) \right\|^2  \\[1.2ex]   
\le \frac{1}{R}\mathbb{E}\left\|\frac{\nabla F_{m}(x_{t-1})}{Mp^{k,t}_{m}}-\frac{\nabla F_{m}(\bar{x}_{k-1})}{Mp^{k,t}_{m}}\right\|^2.   
\end{align*}

Thus, by combining the above two equality, we obtain 
$$\mathbb{E}\|v_t\|^2 \le \\
\frac{1}{R}\mathbb{E}\left\|\frac{\nabla
F_{m}(x_{t-1})}{Mp^{k,t}_{m}}-\frac{\nabla F_{m}(\bar{x}_{k-1})}{M p^{k,t}_{m}}\right\|^2+\|\nabla F(x_{t-1})\|^2 $$

Now, we can do the standard analysis of the SVRG. For $t=1,\ldots,T$, by given all of the randomness before $t$, we then have that 
\begin{align*}
 \mathbb{E}\|x_t-x^*\| 
 &\le\|x_{t-1}-x^*\|^2-2\eta(x_{t-1}-x^*)\mathbb{E}v_t+\eta^2\mathbb{E}\|v_t\|^2 \\[1.2ex]
&\le\|x_{t-1}-x^*\|^2-2\eta(x_{t-1}-x^*)\nabla F(x_{t-1})    \\[1.2ex]
&\hspace{4mm}+\frac{\eta^2}{R}\mathbb{E}\left\| \frac{\nabla F_{m}(x_{t-1})}{M p^{k,t}_{m}}-\frac{\nabla F_{m}(\bar{x}_{k-1})}{M p^{k,t}_{m}}\right\|^2 \\[1.2ex]
&\hspace{4mm} + \eta^2\|\nabla F(x_{t-1})\|^2\\[1.1ex]
&\le \|x_{t-1}-x^*\|^2-2\eta[F(x_{t-1})-F(x^*)]  \\[1.2ex]
&\hspace{3mm}+\frac{\eta^2}{RM^2}\left(\sum_{m=1}^M\|\nabla F_m(x_{t-1})-\nabla F_m(\bar{x}_{k-1})\|\right)^2\\[1.2ex]
&\hspace{4mm}+\eta^2\|\nabla F(x_{t-1})\|^2
\end{align*}
Let denote  $I:=\mathbb{E}\left(\sum_{m=1}^M\|\nabla F_m(x_{t-1})-\nabla F_m(\bar{x}_{k-1})\|\right)^2$.\medskip\\ 
Then we have :

$I\le \mathbb{E} \Big(\sum_{m=1}^M\|\nabla F_m(x_{t-1})-\nabla F_m(x^*)\|$  

\hspace{10mm}$+\sum_{m=1}^M\|\nabla F_m(\bar{x}_{k-1})-\nabla F_m(x^*)\| \Big)^2$ 

$\leq  2 \mathbb{E}\left(\sum_{m=1}^M\|\nabla F_m(x_{t-1})-\nabla F_m(x^*)\|\right)^2$ 

\hspace{10mm} $+ 2 \mathbb{E}\left(\sum_{m=1}^M\|\nabla F_m(\bar{x}_{k-1})-\nabla F_m(x^*)\|\right)^2$ 

$\le 2\mathbb{E}\Big[\sum_{m=1}^M\sqrt{2L_m}$ 

\hspace{10mm}$\sqrt{F_m(x_{t-1})-F_m(x^*)-\nabla F_m(x^*)(\bar{x}_{k}-x^*)}\Big]^2 $

$+  2\mathbb{E}\Big[\sum_{m=1}^M\sqrt{2L_m}$ 

\hspace{10mm}$\sqrt{F_m(\bar{x}_{k-1})-F_m(x^*)-\nabla F_m(x^*)(\bar{x}_{k-1}-x^*)} \Big]^2$ 

$\le 4 \sum_{m=1}^M L_m \times$ 

\hspace{3mm} $\mathbb{E}\left[\sum_{m=1}^MF_m(x_{t-1})-F_m(x^*)-\nabla F_m(x^*)(\bar{x}_{k}-x^*)\right] $

$+4 \sum_{m=1}^ML_m  \times$ 

\hspace{5mm} $\mathbb{E}\left[\sum\limits_{m=1}^MF_m(\bar{x}_{k-1})-F_m(x^*)-\nabla F_m(x^*)(\bar{x}_{k-1}-x^*)\right]$ 

$=4M \sum_{m=1}^M L_m \times  \mathbb{E}\Big[F(x_{t-1})-F(x^*)\Big]$

\hspace{12mm}$+4M \sum_{m=1}^ML_m \times \mathbb{E}\Big[F(\bar{x}_{k-1})-F(x^*)\Big] $\\

In addition, we also have that \medskip\\
$\|\nabla F(x_{t-1})\|^2\le 2\left(\frac{\sum_{m=1}^ML_m}{M}\right)\Big[F(x_{t-1})-F(x^*)\Big]$\\

Thus, we deduce that  \medskip\\
$\mathbb{E}\|x_{t}-x^*\|^2  \le \|x_{t-1}-x^*\|^2$

\hspace{25mm}$-2\eta \Big[F(x_{t-1})-F(x^*)\Big]$

\hspace{25mm}$+2\eta^2 \frac{\sum_{m=1}^ML_m}{M} \big[F(x_{t-1})-F(x^*)\Big]$

\hspace{25mm}$+4\eta^2 \frac{\sum_{m=1}^ML_m}{RM}\Big[F(x_{t-1})-F(x^*)\Big]$

\hspace{25mm}$ +4\eta^2 \frac{\sum_{m=1}^ML_m}{RM} \Big[F(\bar{x}_{k-1})-F(x^*)\Big]$

\hspace{20mm}$=\|x_{t-1}-x^*\|^2-2\eta \Big[F(x_{t-1})-F(x^*) \Big]$ 

\hspace{25mm}$+\eta^2\left(2+\frac{4}{R}\right)\bar{L}\Big[F(x_{t-1})-F(x^*)\Big]$

\hspace{22mm}$\hspace{3mm}+4\eta^2\frac{\bar{L}}{R}\Big[F(\bar{x}_{k-1})-F(x^*)\Big]$

We then  take the sum of the above inequalities for $t=1,\ldots,T$ to obtain

$\mathbb{E}\|x_{T}-x^*\|^2 \le$

\hspace{18mm}$\mathbb{E}\|\bar{x}_{k-1}-x^*\|-2\eta \mathbb{E}\sum\limits_{t=1}^{T}\Big[F(x_{t-1})-F(x^*)\Big] $

\hspace{18mm}$+\eta^2\left(2+\frac{4}{R}\right)\bar{L}\mathbb{E}\sum_{t=1}^{T}\Big[F(x_{t-1})-F(x^*)\Big] $

\hspace{18mm}$+4\eta^2\frac{\bar{L}}{R}T\mathbb{E}\Big[F(\bar{x}_{k-1})-F(x^*)\Big]$

\hspace{10mm}$\le \frac{2}{\gamma}\mathbb{E}\Big[F(\bar{x}_{k-1})-F(x^*)\Big]-2\eta T\mathbb{E}\Big[f(\bar{x}_{k})-f(x^*)\Big]$

\hspace{18mm}$+ \eta^2\left(2+\frac{4}{R}\right)\bar{L}T\mathbb{E}\Big[F(\bar{x}_{k})-F(x^*)\Big] $

\hspace{18mm}$+4\eta^2\frac{\bar{L}}{R}T\mathbb{E}\Big[F(\bar{x}_{k-1})-F(x^*)\Big] $

\hspace{10mm}$= \left[\frac{2}{\gamma}+4\eta^2\frac{\bar{L}}{R}T\right]\mathbb{E}\Big[F(\bar{x}_{k-1})-F(x^*)\Big]$

\hspace{15mm} $-\left[2\eta T-\eta^2\left(2+\frac{4}{R}\right)\bar{L}T\right]\mathbb{E}\Big[F(\bar{x}_{k})-F(x^*)\Big] $

From the above inequality, we deduce that 
\begin{align*}
 \mathbb{E}\Big[F(\bar{x}_{k})-F(x^*)\Big] \le \alpha\mathbb{E}\Big[F(\bar{x}_{k-1})-F(x^*)\Big]     
\end{align*}
where
\begin{align*}
 \alpha= \frac{1}{\gamma\eta T\left[1-\eta\left(1+\frac{2}{R}\right)\bar{L}\right]}+\frac{2\eta\frac{\bar{L}}{R}}{1-\eta\left(1+\frac{2}{R}\right)\bar{L}}   
\end{align*}

\end{proof}

\subsection{Proof of Theorem \ref{thm2}}

\begin{theorem*} 
Given $K,T,R>0$, and $\eta$ small. The iteration $\bar{x}_k$ in \hyperlink{asdsvrg}{$\mathbf{ASD-SVRG}$} converges to the optimal solution $x^*$ linearly in expectation. Moreover, under the condition each of the $n_m = \frac{9}{2} K^2_m \log{\frac{2dM}{\delta}}$, then the following inequality get satisfied:
\begin{align*}
  & \mathbb{E}[F(\bar{x}_{k})-F(x^*)]\le \rho\mathbb{E}[F(\bar{x}_{k-1})-F(x^*)],
\end{align*}
with probability of $1-\delta $ where $\rho$ defined as
\begin{align*}
 \rho= \frac{1}{\lambda T \eta  \big[ 1-\eta  \bar{L} ( \frac{8}{R} +1) \big] }
+\frac{8 \eta \bar{L}  }{R \big[ 1-\eta  \bar{L} ( \frac{8}{R} +1) \big] }.
\end{align*}

\end{theorem*}

\begin{proof}
We have: 
\begin{align*}
  v_t=\frac{1}{R}\sum_{r=1}^R\frac{\nabla F_{m^r}(x_{t-1})-\nabla F_{m^r}(\bar{x}_{k-1})}{M\widetilde{p}_{m^r}}+\nabla F(\bar{x}_{k-1})  
\end{align*}
where $\widetilde{p}$  stands for perturbed distribution as mentioned in lemma \ref{noisydist}. 
We notice $v_t$ is still an unbiased estimate of $\nabla F(x_{t-1})$. 
Following the same argument used in the previous proof: 
\begin{align} \label{normici}
  \mathbb{E}_{\widetilde{p}}\|v_t\|^2\le \frac{1}{R}\mathbb{E}_{\widetilde{p}} \left\|\frac{\nabla F_{m_t}(x_{t-1})-\nabla F_{m_t}(\bar{x}_{k-1})}{M\widetilde{p}_{m_t}}\right\|^2 \\
+  \|\nabla F(x_{t-1})\|^2  \hspace{35mm} \nonumber
\end{align}
To understand the expression above further, we analyse the term inside expectation in detail:
$$\left\| \frac{\nabla F_{m_t}(x_{t-1})-\nabla F_{m_t}(\bar{x}_{k-1})}{M\widetilde{p}_{m_t}} \right\|=$$

$$ \frac{  p_{m_t} }{ \widetilde{p}_{m_t}  }
 \left\| \frac{\nabla F_{m_t}(x_{t-1})-\nabla F_{m_t}(\bar{x}_{k-1})}{M p_{m_t}} \right\| = $$ 

$$ \frac{  p_{m_t} }{ \widetilde{p}_{m_t}  } \frac{ \sum_{m=1}^M\|\nabla F_m(x_{t-1})-\nabla F_m(\bar{x}_{k-1})\| } {M } \leq $$

$$ \frac{ 1 }{ 1-\gamma  } \frac{ \sum_{m=1}^M\|\nabla F_m(x_{t-1})-\nabla F_m(\bar{x}_{k-1})\| } {M }$$

Here, to see the correctness of the inequality, we refer to the discussion after proof of the lemma \ref{noisydist}.
Then returning back to inequality (\ref{normici}) we obtain:\\

$\mathbb{E}_{\hat{p}}\|v_t\|^2\le $

\hspace{15mm}$ \frac{1}{R}\mathbb{E}_{\widetilde{p}} \left[ \left(
\frac{ 1 }{ 1-\gamma  } \frac{ \sum_{m=1}^M\|\nabla F_m(x_{t-1})-\nabla F_m(\bar{x}_{k-1})\| } {M } \right)^2 \right]$

\hspace{15mm}$ +\|\nabla F(x_{t-1})\|^2$

\hspace{10mm}$=\frac{1}{R M^2(1-\gamma)^2}  
\left( \sum\limits_{m=1}^M\|\nabla F_m(x_{t-1})-\nabla F_m(\bar{x}_{k-1})\|  \right)^2$

\hspace{17mm}$\times \mathbb{E}_{\widetilde{p}}
\left[ 1 \right] +\|\nabla F(x_{t-1})\|^2$\\

\hspace{10mm} $\leq \frac{4}{M^2R}\left(\sum\limits_{m=1}^M\|\nabla f_m(x_{t-1})-\nabla f_m(\bar{x}_{k-1})\|\right)^2$

\hspace{17mm}$+\|\nabla F(x_{t-1})\|^2$\\

\hspace{10mm} $\le 16 \frac{1}{R}\bar{L}\Big(\big[F(x_{t-1})-F(x^*)\big]$ 

\hspace{17mm} $+\big[F(\bar{x}_{k-1})-F(x^*)\big]\Big) +\|\nabla F(x_{t-1})\|^2$\\

Here the first equality get satisfied because the term inside expectation is invariant to selected category due to $\widetilde{p}$.
The second inequality is due to the lemma \ref{oneminus} and the last one coming by following the same procedure with the proof of the lemma \ref{lmbck}.
Then, by bounding $\mathbb{E}\|x_T-x^*\|^2$ similar to the proof of the lemma \ref{lmbck}, we obtain:\\
\begin{align*}
\mathbb{E}\|x_T-x^*\|^2 &\le \Big( \frac{2}{\lambda} +\eta^2T\frac{1}{R}\left[16 \bar{L}\right] \Big)
\mathbb{E}[F(\bar{x}_{k-1})-F(x^*)] \\[1.2ex]
&\hspace{-8mm}-\left[2\eta-\eta^2\frac{1}{R}\left[16 \bar{L}+2R\bar{L}\right]\right]T\mathbb{E}[F(x_{k})-F(x^*)].
\end{align*}
Therefore, this inequality implies that 
\begin{align*}
  \mathbb{E}[F(x_k)-F(x^*)]\le \rho \mathbb{E}[F(x_{k-1})-F(x^*)]  
\end{align*}

where 
\begin{align*}
 \rho= \frac{1}{\lambda T \eta  \big[ 1-\eta  \bar{L} ( \frac{8}{R} +1) \big] }
+\frac{8 \eta \bar{L}  }{R \big[ 1-\eta  \bar{L} ( \frac{8}{R} +1) \big] }.
\end{align*}

\end{proof}


\end{document}